\documentclass[journal, 10pt]{IEEEtran}
\usepackage{graphicx} 
\usepackage{subfigure} 
\usepackage{cite}
\usepackage{amsmath}

\usepackage{algorithm}
\usepackage{algorithmic}
\usepackage{hyperref}

\usepackage{mathtools}

\newcommand{\expect}{\operatorname{E}\expectarg}
\DeclarePairedDelimiterX{\expectarg}[1]{[}{]}{%
  \ifnum\currentgrouptype=16 \else\begingroup\fi
  \activatebar#1
  \ifnum\currentgrouptype=16 \else\endgroup\fi
}

\newcommand{\innermid}{\nonscript\;\delimsize\vert\nonscript\;}
\newcommand{\activatebar}{%
  \begingroup\lccode`\~=`\|
  \lowercase{\endgroup\let~}\innermid 
  \mathcode`|=\string"8000
}

\usepackage{amsmath}
\usepackage{amssymb}
\usepackage{amsthm}
\usepackage{enumerate}
\usepackage{color}
\newtheorem{theorem}{Theorem}
\newtheorem{corollary}{Corollary}
\newtheorem{lemma}{Lemma}
\newtheorem{assumption}{Assumption}

\newtheorem{example-non*}{Example}

\newtheorem{remark}{Remark}

\allowdisplaybreaks 

\DeclareMathOperator*{\argmin}{arg\,min}
\DeclareMathOperator*{\argmax}{arg\,max}

\newcommand{\bs}{\boldsymbol}
\newcommand{\mr}{\mathrm}

\ifodd 1
\newcommand{\cem}[1]{{\color{red}(CEM:#1)}}
\else
\newcommand{\cem}[1]{}
\fi

\ifodd 0
\newcommand{\nrev}[1]{{\color{blue}#1}}
\else
\newcommand{\nrev}[1]{#1}
\fi

\ifodd 0
\newcommand{\orev}[1]{{\color{blue}#1}}
\else
\newcommand{\orev}[1]{#1}
\fi

\ifodd 0
\newcommand{\rev}[1]{{\color{blue}#1}}
\else
\newcommand{\rev}[1]{#1}
\fi

\ifodd 0
\newcommand{\com}[1]{{\color{red}(COMMENT:#1)}}
\else
\newcommand{\com}[1]{}
\fi

\ifodd 0
\newcommand{\jinc}[1]{}
\newcommand{\sinc}[1]{{\color{black}#1}} 
\else
\newcommand{\jinc}[1]{{\color{black}#1}} 
\newcommand{\sinc}[1]{}
\fi

\title{Adaptive Ensemble Learning with Confidence Bounds}
\author{\IEEEauthorblockN{Cem Tekin,~\IEEEmembership{Member,~IEEE}, Jinsung Yoon, Mihaela van der Schaar,~\IEEEmembership{Fellow,~IEEE}\\}
\thanks{A preliminary version of this paper will be presented at the HIAI'16 workshop at AAAI'16 \cite{tekin2016adaptive}.}
\thanks{C. Tekin is with the Department of Electrical and Electronics Engineering, Bilkent University, Ankara, Turkey, 06800. Email: cemtekin@ee.bilkent.edu.tr}
\thanks{J. Yoon and M. van der Schaar are with the Department of Electrical Engineering, UCLA, Los Angeles, CA, 90095. Email: jsyoon0823@gmail.com, mihaela@ee.ucla.edu}
\thanks{The work of C. Tekin has been supported by The Scientific and Technological Research Council of Turkey (TUBITAK) under 2232 Fellowship 116C043. The work of J. Yoon and M. van der Schaar has been supported by NSF EECS 1407712.}}

\begin{document}
\maketitle
\begin{abstract}
Extracting actionable intelligence from distributed, heterogeneous,
correlated and high-dimensional data sources requires run-time processing
and learning both locally and globally. In the last decade, a large
number of meta-learning techniques have been proposed in which local
learners make online predictions based on their locally-collected
data instances, and feed these predictions to an ensemble learner,
which fuses them and issues a global prediction. However, most of
these works do not provide performance guarantees or, when they do,
these guarantees are asymptotic. None of these existing works provide
confidence estimates about the issued predictions or rate of learning
guarantees for the ensemble learner. In this paper, we provide a systematic
ensemble learning method called Hedged Bandits, which comes with both
long run (asymptotic) and short run (rate of learning) performance
guarantees. \nrev{Moreover, our approach yields performance guarantees with respect to the optimal local prediction strategy, and is also able to adapt its predictions in a data-driven manner.} We illustrate the performance of Hedged Bandits in the
context of medical informatics \nrev{and show that it outperforms numerous online and offline ensemble learning methods.}
\end{abstract}


\section{Introduction}\label{sec:intro}

Huge amounts of data streams are now being produced by more and more
sources and in increasingly diverse formats: sensor readings, physiological
measurements, GPS events, network traffic information, documents,
emails, transactions, tweets, audio files, videos etc. These streams
are then mined in real-time to provide actionable intelligence for
a variety of applications: patient monitoring \cite{simons2008consumer},
recommendation systems \cite{cao2007intelligent}, social networks
\cite{beach2008whozthat}, targeted advertisement \cite{swix2004method},
network security \cite{canzian2015timely, eskicioglu2001overview}, medical diagnosis
\cite{arsanjani2013improved} etc. Hence, online data mining algorithms
have emerged that analyze the correlated, high-dimensional and dynamic
data instances captured by one or multiple heterogeneous data sources,
extract actionable intelligence from these instances and make decisions
in real-time. To mine these data streams, the following questions
need to be answered online, for each data instance: Which processing/prediction/decision
rule should a {\em local learner} (LL) select? How should the LLs
adapt and learn their rules to maximize their performance? How should
the processing/predictions/decisions of the LLs be combined/fused
by a meta-learner to maximize the overall performance?

Existing works on meta-learning \cite{zheng2011attribute, littlestone1989weighted,freund1995desicion,canzian2015timely} have aimed to provide solutions to these
questions by designing {\em ensemble learners} (ELs) that fuse
the predictions%
\footnote{Throughout this paper the term prediction is used to denote a variety
of tasks from making predictions to taking actions.%
} made by the LLs into global predictions. A majority of the literature
treats the LLs as black box algorithms, and proposes various fusion
algorithms for the EL with the goal of issuing predictions that are
at least as good as the best LL in terms of prediction accuracy. In
some of these works, the obtained result holds for any arbitrary sequence
of data instance-label pairs, including the ones generated by an adaptive
adversary. However, the performance bounds proved for the EL in these
papers depend on the performance of the LLs. In this work, we go one
step further and study the joint design of learning algorithms for
both the LLs and the EL. 
\nrev{Our approach also differs from {\em empirical risk minimization} (ERM) based approaches \cite{vapnik1992principles, kivinen2004online}. Firstly, most of the literature on ERM is concerned with finding the best prediction rule on average. We depart from this approach and seek to find the best context-dependent prediction rule. Secondly, data is not available a priori in our model. Predictions are made on-the-fly based on the prediction rules chosen by the learning algorithm. This results in a trade-off between exploration and exploitation, which is not present in ERM.}

In this paper, we present a novel learning method which continuously
learns and adapts the parameters of both the LLs and the EL, after
each data instance, in order to achieve strong performance guarantees
- both confidence bounds and regret bounds. We call the proposed method
\emph{Hedged Bandits} (HB). The proposed system consists of a \rev{new} contextual
bandit algorithm for the LLs and \rev{two new variants of the} Hedge algorithm \cite{freund1995desicion}
for the EL. The proposed method is able to exploit the adversarial
regret guarantees of Hedge and the data-dependent regret guarantees
of the contextual bandit algorithm to derive regret
bounds for the EL. 
One proposed variant of the Hedge algorithm does not require the knowledge of time horizon $T$ and achieves the $O(\sqrt{T\log M})$ on regret uniformly over time, where $M$ is the number of
LLs. The other variant uses the context/side information provided to the EL to fuse the predictions of the LLs. 

The contributions of this paper are: 
\begin{itemize}
\item We propose two variants of the Hedge algorithm \cite{freund1995desicion}. The first variant, which is called Anytime Hedge (AH), is a parameter-free Hedge algorithm \cite{auer2002adaptive, chaudhuri2009parameter, de2014follow, cesa2006prediction}. 
We prove that AH enjoys the same order of regret as the original Hedge \cite{freund1995desicion}. The second variant, which is called Contextual Hedge (CH), is novel and uses the context information provided to the EL when fusing the LLs' predictions. Since the sequence of context arrivals to the EL are not known in advance, CH utilizes AH to learn the best LL for each context.
\item We propose a new index-based learning rule for each LL, called Instance-based Uniform Partitioning (IUP). We prove an optimal regret bound for IUP, which holds for any sequence of data instance arrivals to the LL, and hence, also in expectation.  
\item We prove confidence bounds for each LL with respect to the optimal data-dependent prediction rule of that LL. 
\item Using the regret bounds proven for each LL and the EL, we prove a regret bound for the EL with respect to the optimal data-dependent prediction rule. 
\item We numerically compare IUP, AH and CH with state-of-the-art machine learning methods in the context of medical informatics and show the superiority of the proposed methods. 
\end{itemize}

\section{Problem Description}\label{sec:ab}

This section describes the system model and introduces the notation. 
$\mr{I}(\cdot)$ is the indicator function, $\mr{E}[\cdot]$ is the expectation operator.
$\mr{E}_{P}[\cdot]$ denotes the expectation of a random variable with respect to distribution $P$. Given a set ${\cal S}$, $\Delta( {\cal S} )$ denotes the set of probability distributions over ${\cal S}$ and $|{\cal S}|$ denotes the cardinality of ${\cal S}$. 
For a scalar or vector $z(t)$ indexed by $t \in \mathbb{N}^+ := \{1,2,\ldots\}$, $\bs{z}^T := (z(1),\ldots,z(T))$. \nrev{Given a vector $\bs{v}$, $\bs{v}_{-i}$ is the vector formed by the components of $\bs{v}$ except the $i$th component.}
Random variables are denoted by uppercase letters. Realizations of random variables are denoted by lowercase letters.

The system model is given in Fig. \ref{fig:ensemble}. There
are $M$ LLs indexed by the set ${\cal M}:=\{1,2,\ldots,M\}$. Each
LL receives streams of data instances, sequentially, over discrete
time steps $t \in \{ 1,2,\ldots \}$. The instance received by LL $i$ at
time $t$ is denoted by $X_{i}(t)$. Without loss of generality, we
assume that $X_{i}(t)$ is a $d_{i}$-dimensional vector in ${\cal X}_{i}:=[0,1]^{d_{i}}$.%
\footnote{The unit hypercube is just used for notational simplicity. Our methods
can easily be generalized to arbitrary bounded, finite dimensional
data spaces, including spaces of categorical variables.%
}
\rev{Let ${\cal X} := \prod_{i \in {\cal M}} {\cal X}_i$ denote the {\em joint data instance set}.}

\begin{figure}
\begin{centering}
\includegraphics[width=1\columnwidth]{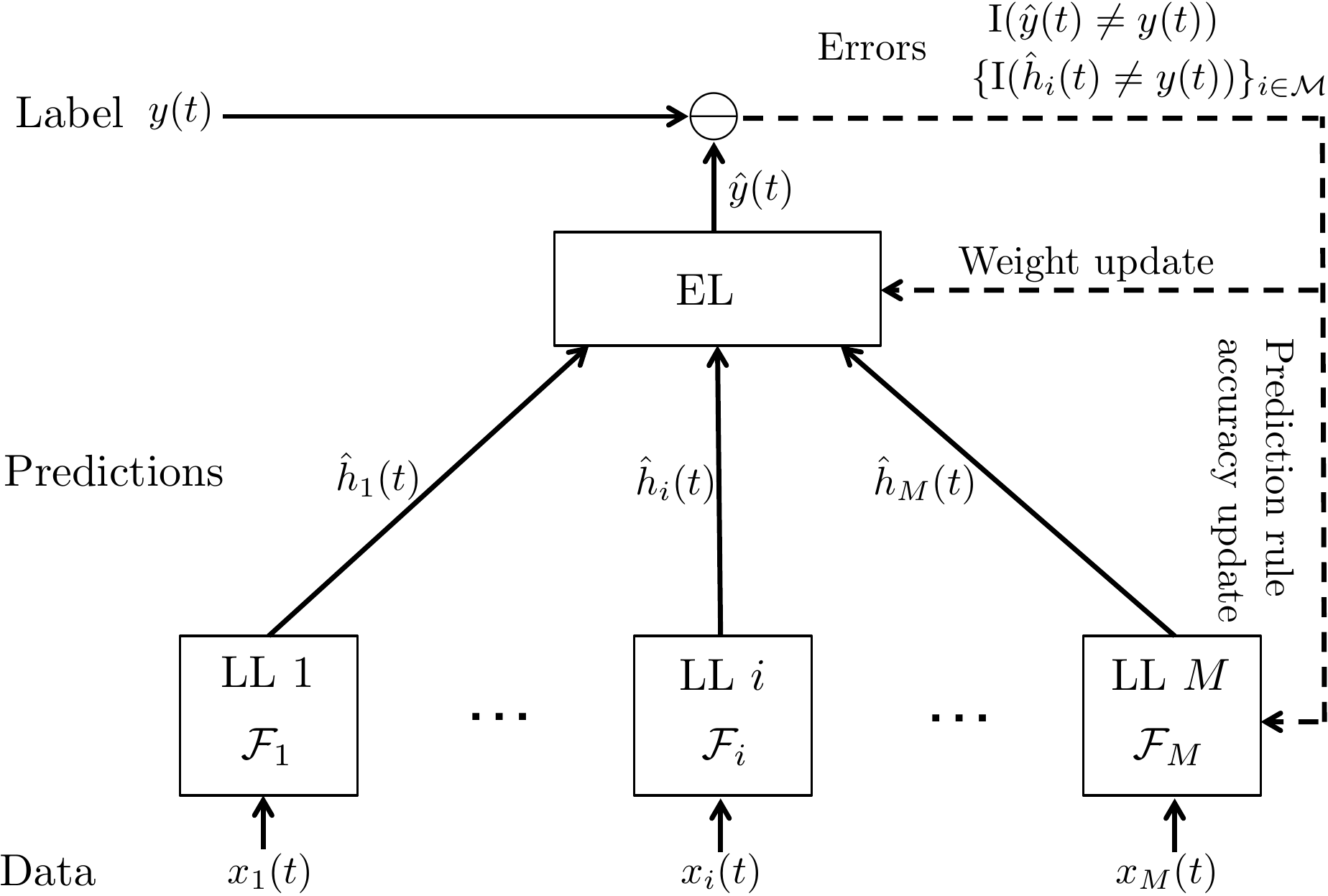}
\caption{Block diagram of the HB. The flow of information towards the EL is illustrated via a tree graph, where the LLs are the leaf nodes. After observing the instance, each LL selects one of its prediction rules to produce a prediction, and sends its prediction to the EL which makes the final prediction. Then, both
the LLs and the EL update their prediction policies based on the received
feedback $y(t)$. Note that the EL only observes the predictions $\hat{\bs{h}}(t)$ of the LLs but not their instances $\bs{x}(t)$.
}
\label{fig:ensemble} 
\end{centering}
\end{figure}

The collection of data instances at time $t$ is denoted by $\boldsymbol{X}(t)=\{X_{i}(t)\}_{i\in{\cal M}}$.
For example, $\boldsymbol{X}(t)$
can include in a medical diagnosis application real-valued features such as lab test results; discrete
features such as age and number of previous conditions; and categorical
features such as gender, smoker/non-smoker, etc. In this example each
LL corresponds to a (different) medical expert. The true label at time $t$ is denoted by $Y(t)$, which is a random variable that takes values in the finite label set ${\cal Y}$. 
Let $J$ denote the joint distribution of $( \bs{X}(t), Y(t) )$. It is assumed that $\{ ( \bs{X}(t), Y(t) ) \}_{t=1}^T$ is i.i.d. Let $J^i_{x_i}$ denote the conditional distribution of $( \bs{X}_{-i}(t), Y(t) )$ given $X_i(t) = x_i$.

The set of prediction rules of LL $i$ is denoted by ${\cal F}_{i}$.
\nrev{For instance, a prediction rule can be a classifier such as an SVM with polynomial kernel, a neural network or a decision tree.}
 Let ${\cal F} := \cup_{i \in {\cal M}} {\cal F}_i$ denote the set of all prediction rules. The prediction produced by $f \in {\cal F}_i$ given context $x_i \in {\cal X}_i$ is denoted by $Y_f(x_i)$. $Y_f(x_i)$ is a random variable whose distribution is given by $Q_f (x_i)$, where $Q_f: {\cal X}_i \rightarrow \Delta({\cal Y})$.
Prediction of $f \in {\cal F}_i$ at time $t$ is denoted by $\hat{Y}_f(t) := Y_f(X_i(t))$.
\nrev{Let $\bs{Z}(t) := ( \bs{X}(t), Y(t), \{ \hat{Y}_f(t) \}_{f \in {\cal F}} )$. Then, 
$ \{ \bs{Z}(t) \}_{t=1}^T$ is an i.i.d. sequence.}
\rev{Realizations of the random variables $X_i(t)$, $Y(t)$ and $\hat{Y}_f(t)$ are denoted by $x_i(t)$, $y(t)$ and $\hat{y}_f(t)$, respectively.}
The accuracy of prediction
rule $f\in{\cal F}_{i}$ for a data instance $x \in {\cal X}_i$ is given as 
\begin{align}
\pi_{f}(x):= \expect*{ \mr{I} ( \hat{Y}_f(t) = Y(t)) | X_{i}(t) = x  }  .  \notag
\end{align}

\rev{LL $i$ operates as follows: It first observes $x_{i}(t)$, and then selects a prediction rule $a_{i}(t)\in{\cal F}_{i}$.}
The selected prediction rule produces a prediction $\hat{h}_{i}(t) = \hat{y}_{a_{i}(t)}( t )$.\footnote{\rev{Without loss of generality we assume that only the selected prediction rule produces a prediction. For instance, in big data stream mining, the LL may be resource constrained and require to make timely predictions. The LL in this setting is constrained to activate only one of its prediction rules for each data instance. Moreover, observing the predictions of more than one prediction rule will result in faster learning. Hence, all our performance bounds will still hold when the LL observes the predictions of all of its prediction rules.}}
Then, all LLs send their predictions $\hat{\bs{h}}(t):=\{\hat{h}_{i}(t)\}_{i\in{\cal M}}$
to the EL, which combines them to produce a final prediction $\hat{y}(t)$. We assume that
the true label $y(t)$ is revealed after the final prediction,
by which the LLs and the EL can update their prediction rule selection
strategy, which is a mapping from the history of past observations,
decisions, and the current instance to the set of prediction rules.
We call $r_f(t) = \mr{I} (\hat{y}_f(t) = y(t) )$ the reward of prediction rule $f$, $v_i(t) := \mr{I} (\hat{h}_i(t) = y(t) )$ the reward of LL $i$ and $r_{\text{EL}}(t) = \mr{I} (\hat{y}(t) = y(t) )$ the reward of the EL at time $t$.
\rev{Random variables that correspond to the realizations $a_i(t)$, $\hat{h}_i(t)$, $r_f(t)$, $v_i(t)$ and $r_{\text{EL}}(t)$ are denoted by $A_i(t)$, $\hat{H}_i(t)$, $R_f(t)$, $V_i(t)$ and $R_{\text{EL}}(t)$, respectively.}

In our setup each LL is only required to observe its own data instance
and know its own prediction rules. However, the accuracy of the prediction
rules is unknown and data dependent. The EL does not know anything about the
instances and prediction rules of the LLs.\footnote{We consider the case when the EL has access to a subset of the features of the instances in Section \ref{sec:extensions}, and propose a learning algorithm for this case.}
We assume that the accuracy of a prediction rule obeys the following
H\"{o}lder rule, which represents a similarity measure between different
data instances. 
\begin{assumption}\label{ass:1} 
There exists $L>0$,
$\alpha>0$ such that for all $i\in{\cal M}$, $f\in{\cal F}_{i}$,
and $x,x'\in{\cal X}_{i}$, we have 
\begin{align*}
|\pi_{f}(x)-\pi_{f}(x')|\leq L||x-x'||^{\alpha}.
\end{align*}
\end{assumption}

\rev{We assume that $\alpha$ is known by the LLs.} Going back
to our medical informatics example, we can interpret Assumption \ref{ass:1}
as follows. If the lab tests, symptoms and
demographic information of two patients are similar, it is expected
that they have the same underlying medical condition, and hence, the (diagnosis)
prediction should be similar for these two patients. 

\section{Performance Metrics: Regret}\label{sec:regret}

In this section, we introduce several performance metrics to assess the performance of the learning algorithms of the LLs and the EL. First, we define the performance measures for the LLs. We start by defining the optimal prediction rules and {\em local oracles} (LOs) that implement these prediction rules. 
Let $f_{i}^{*}(x)$ be the optimal prediction rule of LL $i$ for an instance $x \in {\cal X}_{i}$, which is given by $f^*_i (x) \in  \argmax_{ f\in{\cal F}_{i} }  \pi_{f}(x)$. The accuracy of $f^*_i(x)$ is denoted by $\pi^*_i(x) := \pi_{f^*_i (x)}(x)$.

LO $i$ knows  $\{ \pi_{f}(\cdot) \}_{f \in {\cal F}_i}$ perfectly. At each time step $t$ it observes $x_i(t)$ and then selects $f^*_i(x_i(t))$ to make a prediction. 
Since LL $i$ does not know $\{ \pi_{f}(\cdot) \}_{f \in {\cal F}_i}$ a priori, we would like to measure how well it performs with respect to LO $i$. 
For this, we define the {\em data-dependent regret} of LL $i$ with respect to LO $i$ as 
\begin{align}
\text{Reg}_{i}(T) & :=    \sum_{t=1}^{T} R_{ f^*_i(X_i(t)) }(t)-
\sum_{t=1}^{T}  R_{ A_i(t) } (t)   . \notag 
\end{align}
The strategy of LO $i$ only depends on \rev{$\bs{X}^T_i = (X_i(1), \ldots, X_i(T))$}. Thus, we would like to measure how well LL $i$ performs given $\bs{X}^T_i$. For this, we define the {\em conditional regret} of LL $i$ as 
\begin{align}
\text{Reg}_{i}(T | \bs{X}_i^T)  &:= \expect{ \text{Reg}_{i}(T) | \bs{X}_i^T }  .   \label{eqn:experregret}
\end{align}
\rev{The algorithm we propose in Section \ref{sec:IUP} {\em almost surely} (a.s.) upper bounds the conditional regret with a deterministic sublinear function of time.}
The {\em expected regret} of LL $i$ is defined as 
\begin{align}
\overline{\text{Reg}}_i(T)  &:= \expect{ \text{Reg}_{i}(T) }  \notag \\
& =  \expect*{ \expect{ \text{Reg}_{i}(T) | \bs{X}_i^T } }
=  \expect{ \text{Reg}_{i}(T | \bs{X}_i^T) }  . \notag
\end{align}
This implies that a deterministic upper bound on $\text{Reg}_{i}(T | \bs{X}_i^T) $ that holds a.s. also holds for $\overline{\text{Reg}}_i(T)$. 

Next, we define the performance measures for the EL. 
Consider any realization $\{ \bs{v}^T_i \}_{i \in {\cal M}}$ of the random reward sequence $\{ \bs{V}^T_i \}_{i \in {\cal M}}$ of the LLs. 
The best LL for this realization is defined as $I_b$, where
$I_b \in \argmax_{ i \in {\cal M} } \sum_{t=1}^T v_i(t)$.
In Section \ref{sec:AH}, we propose a learning algorithm for the EL, whose total reward is close to the total reward of $I_b$ for any realization $\{ \bs{v}^T_i \}_{i \in {\cal M}}$. To measure the distance between total rewards, we define the {\em pseudo-regret} of the EL given $\{ \bs{v}^T_i \}_{i \in {\cal M}}$ as
\begin{align}
 \text{Reg}_{\text{EL}}(T) 
&:=  \sum_{t=1}^T v_{I_b}(t) 
- \expect*{ \sum_{t=1}^T  R_{\text{EL}}(t) }
\label{eqn:pseudoregret}
\end{align}
\rev{where the expectation is taken with respect to the randomization of the EL. In Section \ref{sec:AH} we bound $\text{Reg}_{\text{EL}}(T)$ by a sublinear function of $T$, which implies that $\lim_{T \rightarrow \infty} \text{Reg}_{\text{EL}}(T) / T = 0$.}
$\text{Reg}_{\text{EL}}(T)$ compares the performance of the EL with the best LL, which makes it a relative performance measure. \rev{This is the standard approach taken in prior works in ensemble learning \cite{freund1995desicion,auer2}.} Since LLs themselves are learning agents, $\text{Reg}_{\text{EL}}(T)$ depends on the learning algorithms used by the LLs. Next, we propose a benchmark for the performance measure of the EL that is independent of the learning algorithms used by the LLs.

The \rev{optimal} LO denoted by $i^*$, is given as
$i^{*}  \in \argmax_{i\in{\cal M}}
\expect{ \sum_{t=1}^{T}   R_{f^*_i (X_i(t))}(t) }$.
\rev{LO $i^*$'s total predictive accuracy is greatest among
all LOs.}
On the other hand, the best LL in expectation is defined as 
$i^*_b  \in \argmax_{i \in {\cal M} }  \expect*{ \sum_{t=1}^T R_{A_i(t)}(t) }$.
We would like to emphasize the fact that the expected reward of LL $i$ depends on the learning algorithm used by the LL, while the expected reward of LO $i$ is the optimal that can be achieved given the prediction rules in ${\cal F}_i$. Hence, the latter upper bounds the former. This implies that 
$\expect{ \sum_{t=1}^T R_{A_{i^*_b}(t)}(t) }
\leq \expect{ \sum_{t=1}^T R_{f^*_{i^*} (X_{i^*}(t))}(t) }$.
As an absolute measure of performance we define the {\em expected regret} of the EL as 
\begin{align}
\hspace{-0.1in} \overline{\text{Reg}}_{\text{EL}}(T) \hspace{-0.05in} :=
& \expect*{  \sum_{t=1}^{T} \hspace{-0.05in} \orev{ R_{ f^*_{i^*} ( X_{i^*}(t) )} ( t ) } }
\hspace{-0.05in} -  \expect*{ \sum_{t=1}^{T} \hspace{-0.05in}  R_{\text{EL}}(t) }  \label{eqn:ensebleregret}
\end{align}
which compares the EL with the best LO in terms of the expected reward. 

Our goal is to jointly design algorithms for the LLs and the
EL that minimize the learning loss (i.e. the growth rate of $ \overline{\text{Reg}}_{\text{EL}}(T)$).
This can be viewed equivalently as maximizing the learning speed/rate of the LL and the EL algorithms. 
We will prove in the Section \ref{sec:analysis} a sublinear upper bound on $\overline{\text{Reg}}_{\text{EL}}(T)$, meaning that the proposed algorithms have a provably fast
rate of learning, and the {\em average regret} $\overline{\text{Reg}}_{\text{EL}}(T)/T$
of the proposed algorithms converges asymptotically to $0$.
\nrev{A learning algorithm that achieves sublinear regret guarantees that (in expectation) the number of prediction errors it makes is in the order of that of the optimal LO, which knows the accuracies of the prediction rules for each instance in advance.}

\section{An Instance-based Uniform Partitioning Algorithm for the LLs}\label{sec:IUP}

Each LL uses the {\em Instance-based Uniform Partitioning} (IUP)
algorithm given in Fig. \ref{fig:CLUP}. IUP is designed to exploit
the similarity measure given in Assumption \ref{ass:1} when learning
the accuracies of the prediction rules. Basically, IUP partitions ${\cal X}_{i}$
into a finite number of equal sized, identically shaped, non-overlapping
sets, whose granularities determine the balance between approximation
accuracy and estimation accuracy: increasing the size of
a set in the partition results in more past instances falling within
that set, which positively affects the estimation accuracy, but also allows more dissimilar instances
to lie in the same set, which negatively affects the approximation
accuracy. IUP strikes this balance by adjusting the granularity of
the data space partition based on the information contained within
the similarity measure (Assumption \ref{ass:1}) and the time horizon
$T$.%
\footnote{The doubling trick \cite{cesa1997use} allows any learning algorithm $\Gamma$
  that requires the time horizon as an input to run efficiently (with the same time order of regret) without the knowledge of the time horizon. With the doubling trick, time is partitioned into multiple phases ($j=1,2,\ldots$) with doubling lengths ($T_{1},T_{2},\ldots$). For instance, if the first phase is set to last for $\hat{T}$ time steps, then the length of the $j$th phase is equal to $2^{j-1}\hat{T}$ time steps. In each phase $j$, an independent instance of the original learning algorithm $\Gamma$, denoted by $\Gamma_{j}$, is run from scratch, without using any information available from the previous phases. With the doubling trick, $\Gamma_{j}$'s time horizon input is set to $2^{j-1}\hat{T}$. When we run IUP for LL $i$
  with the doubling trick, the only modification that is needed is to set the partitioning parameter of phase $j$ to $m_{i}=\lceil(2^{j-1}\hat{T})^{1/(2\alpha+d_{i})}\rceil$.}

Let $m_{i}$ be the {\em partitioning parameter} of LL $i$, which
is used to partition $[0,1]^{d_{i}}$ into $m_{i}^{d_{i}}$
identical hypercubes. This partition is denoted by ${\cal P}_{i}$.%
\footnote{Instances laying at the edges of the hypercubes can be assigned to
one of the hypercubes in a random fashion without affecting the derived
performance bounds.%
} IUP estimates the accuracy of each prediction rule for each set (hypercube)
$p\in{\cal P}_{i}$, separately, by only using the
past history from instance arrivals that fall into hypercube $p$.
For each LL $i$, IUP keeps and updates the following parameters during
its operation: 
\begin{itemize}
\item $N_{f,p}^{i}(t)$: Number of times an instance arrived to hypercube
$p\in{\cal P}_{i}$ and prediction rule $f$ of LL $i$ is used
to make the prediction \rev{prior to time $t$}. 
\item $\hat{\pi}_{f,p}^{i}(t)$: Sample mean accuracy of prediction rule
$f \in {\cal F}_i$ \rev{at time $t$}. 
\end{itemize}
An illustration of the partitions used by IUP for each LL is given in Fig \ref{fig:IUPpartition}.
IUP strikes the balance between exploration and exploitation by keeping the following set of indices for each $p \in {\cal P}_i$ and $f \in {\cal F}_i$:\footnote{When $N^i_{f,p}(t) =0$, we set $g^i_{f,p}(t)$ to $+ \infty$.}
\begin{align}
\hspace{-0.1in} g^i_{f,p}(t) = \hat{\pi}^i_{f,p}(t) + \hspace{-0.05in} \sqrt{ \frac{2}{N^i_{f,p}(t)} 
(1 \hspace{-0.05in} + \hspace{-0.05in} 2 \log ( 2 |{\cal F}_i| m_i^{d_i} T^{\frac{3}{2}} ) ) }  .  \label{eqn:UCBindex}
\end{align}
\nrev{The second term in \eqref{eqn:UCBindex} is an inflation term that decreases with the square root of $N^i_{f,p}(t)$. The $(1+2\log(2|{\cal F}_{i}|m_{i}^{d_{i}}T^{3/2})$ term is a normalization constant that is required for the regret analysis in Theorem \ref{thm:CUPregret}. These types of indices are commonly used in online learning \cite{auer} to tradeoff exploration and exploitation.}

At the beginning of time step $t$, LL $i$ observes $x_i(t)$, and identifies the hypercube $p_i(t) \in {\cal P}_i$ that contains $x_i(t)$. Then, it selects
$a_i(t) \in \argmax_{ f \in {\cal F}_i }  g^i_{f,p_i(t)}(t)$
and predicts $\hat{h}_i(t) = \hat{y}_{a_i(t)}(t)$.  
The second term of the index reflects the uncertainty in the estimated value $\hat{\pi}^i_{f,p}(t)$. It decreases as more observations are gathered from prediction rule $f$ for data instances that lie in $p$. Hence, $g^i_{f,p}(t)$ serves as an optimistic estimate of the accuracy of $f$ for data instances in $p$. 
LL $i$ explores when $a_i(t) \notin \argmax_{f \in {\cal F}_i}  \hat{\pi}^i_{f,p_i(t)}(t)$, and exploits when $a_i(t) \in \argmax_{f \in {\cal F}_i}  \hat{\pi}^i_{f,p_i(t)}(t)$. In exploration, it chooses a prediction rule with suboptimal estimated accuracy and high uncertainty, while in exploitation it chooses the prediction rule with the highest estimated accuracy. 
\rev{In Section \ref{sec:analysis}, we will show that the choice of the index in \eqref{eqn:UCBindex} results in optimal learning.}

\begin{figure}
\begin{centering}
\includegraphics[width=0.8\columnwidth]{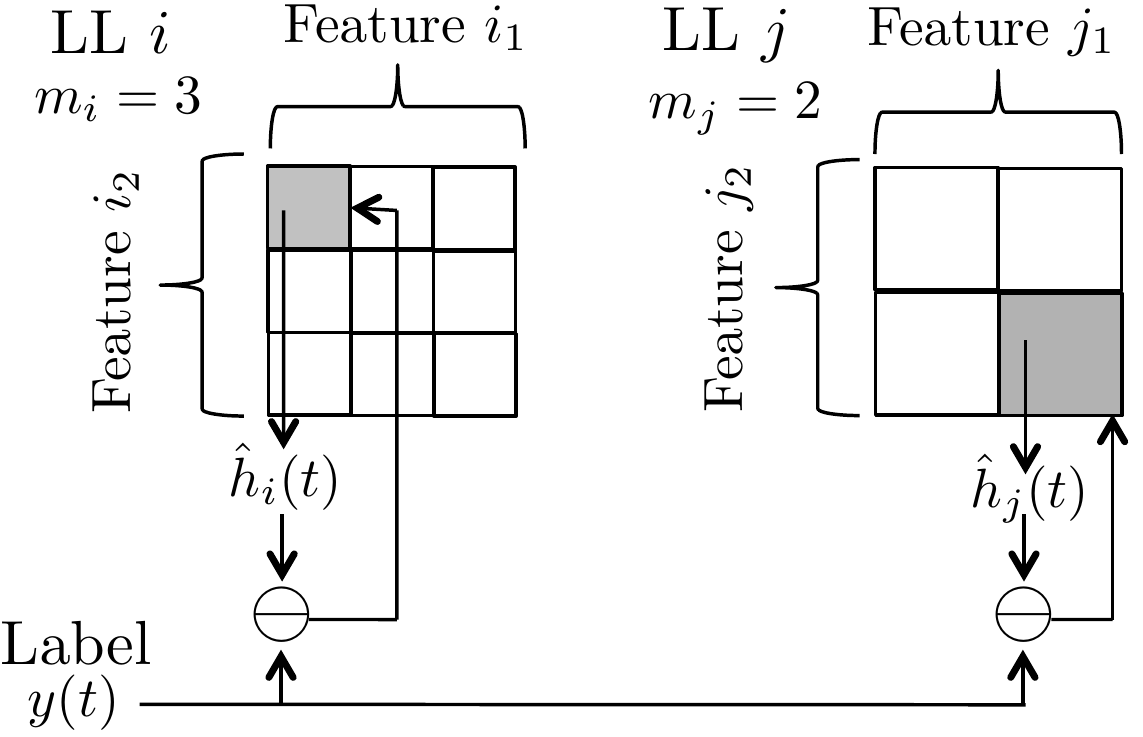}
\caption{Illustration of different partitions used by IUP for LLs $i$ and $j$. Accuracy parameter is updated for the shaded sets in the partitions, which contains the current feature vector.
\com{Change $\hat{y}_i(t)$ to $\hat{u}_i(t)$ in the figure.}
}
\label{fig:IUPpartition} 
\end{centering}
\end{figure}

\begin{figure}[htb]
\fbox{%
\begin{minipage}[c]{0.95\columnwidth}%
{\fontsize{9}{9}\selectfont \flushleft{IUP for LL $i$:} 
\begin{algorithmic}
\STATE{Input: $T$, $m_{i}$, $d_i$}
\STATE{Initialize sets: Create partition ${\cal P}_{i}$ of $[0,1]^{d_i}$ into $m_i^{d_i}$
identical hypercubes}
\STATE{Initialize counters: $N_{f,p}^{i}=0,\forall f\in{\cal F}_{i},p\in{\cal P}_{i}$, $t=1$} 
\STATE{Initialize estimates: $\hat{\pi}_{f,p}^{i} = 0$, $\forall f \in {\cal F}_{i}$,
$p\in{\cal P}_{i}$ } 
\WHILE{$t\geq1$} 
\STATE{Find the set $p^{*} = p_i(t) \in {\cal P}_{i}$ that $x_{i}(t)$ belongs to}
\STATE{Compute the index for each $f \in {\cal F}_i$:}
\FOR{ $f \in {\cal F}_i$ }
\IF{ $N^i_{f,p^*} > 0$ }
\STATE{
$\hspace{-0.2in} g^i_{f,p^*} = \hat{\pi}^i_{f,p^*} + \sqrt{\frac{2}{N^i_{f,p^*}} (1 + 2 \log ( 2 |{\cal F}_i| m_i^{d_i} T^{\frac{3}{2}} ) )  }$ }
\ELSE
\STATE{ $\hspace{-0.2in} g^i_{f,p^*} = +\infty$  }
\ENDIF
\ENDFOR
\STATE{Select $a_i = \argmax_{f \in {\cal F}_i} g^i_{f,p^*}$ (break ties randomly)}
\STATE{Predict \rev{$\hat{h}_{i}(t) = \hat{y}_{a_i}(t)$} }
\STATE{Observe the true label $y(t)$ and the reward $v_i(t)=\mathrm{I}(\hat{h}_{i}(t)=y(t))$} 
\STATE{$\hat{\pi}_{a_{i},p^{*}}^{i} \leftarrow 
( \hat{\pi}_{a_{i},p^{*}}^{i} N_{a_{i},p^{*}}^{i} + v_i(t) ) /
( N_{a_{i},p^{*}}^{i} + 1 )$}
\STATE{\rev{$N_{a_{i},p^{*}}^{i} \leftarrow N_{a_{i},p^{*}}^{i}  + 1$}} 
\STATE{\rev{$t \leftarrow t+1$}} 
\ENDWHILE
\end{algorithmic} 
} %
\end{minipage}}\protect\protect\caption{Pseudocode of IUP for LL $i$.}
\label{fig:CLUP} 
\end{figure}

\section{Anytime Hedge Algorithm for the EL}\label{sec:AH}

\nrev{In this section, we consider a parameter-free variant of the Hedge algorithm, called the {\em Anytime Hedge} (AH). Hedge \cite{freund1995desicion} is an algorithm that uses the exponential weights update rule. It achieves $O(\sqrt{T})$ regret under the prediction with expert advice model. In this model, the goal is to compete with the best expert given a pool of experts.
Hedge takes as input a parameter $\eta$, that is called the {\em learning rate}. The regret of Hedge is minimized when $\eta$ is carefully selected according to the time horizon $T$.
} 

Unlike the original Hedge,
AH does not require a priori knowledge of the time horizon. 
The EL uses AH to produce the final prediction $\hat{y}(t)$.%
\footnote{We decided to use AH as the ensemble learning algorithm due to
its simplicity and regret guarantees. In practice, AH can be replaced
with other ensemble learning algorithms. For instance, we also evaluate
the performance when LLs use IUP and the EL uses Weighted Majority
(WM) algorithm \cite{littlestone1989weighted} in the numerical results section. Unlike AH, WM uses $q_i(t)$ as the weight of the prediction of LL $i$. It sets the weight of $y \in {\cal Y}$ to be $w_y(t) = \sum_{i: \hat{h}_i(t) = y} q_i(t)$ and predicts $\hat{y}(t) \in \argmax_{y \in {\cal Y}} w_y(t)$. 
}
\rev{Although, numerous parameter-free variants of Hedge are introduced in prior works \cite{auer2002adaptive,chaudhuri2009parameter,de2014follow,cesa2006prediction}, to the best of our knowledge the regret analysis for AH is new. Specifically, in Theorem 2.3 of \cite{cesa2006prediction}, regret bound for a parameter-free Exponentially Weighted Average Forecaster is derived. However, it is assumed that (i) the prediction of the EL is a deterministic weighted average of the predictions of the LLs, and (ii) the space of predictions and the loss functions are convex. In contrast to this, in our setting (i) the prediction of the EL is probabilistic, and (ii) the space of prediction is a finite set ${\cal Y}$ and the loss functions $\mr{I} (\hat{h}_{i}(t)\neq y(t))$ and $\mr{I}(\hat{y}(t)\neq y(t))$ are indicator functions.}

\begin{figure}[htb]
\fbox{%
\begin{minipage}[c]{0.95\columnwidth}%
{\fontsize{9}{9}\selectfont \flushleft{Anytime Hedge (AH)} 
\begin{algorithmic}
\STATE{Input: A non-increasing sequence of positive real numbers $\{ \eta(t) \}_{t \in \mathbb{N}^+}$}
\STATE{Initialization: $L_{i}(0) = 0$ for $i \in {\cal M}$, $t=1$}
\WHILE{$t\geq1$} 
\STATE{Receive predictions of LLs: $\hat{\bs{h}}(t)$}
\STATE{Choose the LL $I(t)$ to follow according to the distribution $\bs{q}(t) := (q_1(t),\ldots, q_M(t))$ where
\begin{align}
q_{i}(t) = \frac{ \exp( -\eta(t) L_{i}(t-1) )} { \sum_{j=1}^M \exp ( - \eta(t) L_{j}(t-1) ) }      \notag
\end{align}
}
\STATE{Predict $\hat{y}(t) = \hat{h}_{I(t)}(t)$}
\STATE{Observe the true label $y(t)$}
\STATE{Receive the reward $r_{\text{EL}}(t) = \mr{I} (\hat{y}(t) = y(t))$ and observe losses of all LLs: $l_{i}(t) := \mr{I} (\hat{h}_i(t) \neq y(t))$ for $i \in {\cal M}$}
\STATE{Set $L_{i}(t) = L_{i}(t-1) + l_{i}(t)$}
\STATE{$t \leftarrow t+1$} 
\ENDWHILE
\end{algorithmic} 
} %
\end{minipage}}\protect\protect\caption{Pseudocode of AH.}
\label{fig:Ahedge} 
\end{figure}

AH keeps a cumulative loss/error vector $\bs{L}(t)=(L_{1}(t),\ldots,L_{M}(t))$,
where $L_{i}(t)$ denotes the number of prediction errors made by LL $i$ by the end of time step $t$. After observing
$\hat{\bs{h}}(t)$, AH samples its final prediction $\hat{Y}(t)$ from this set
according to probability distribution $\bs{q}(t) = (q_1(t), \ldots, q_M(t))$, where 
\begin{align}
\mr{Pr}(\hat{Y}(t)=\hat{h}_{i}(t)) = q_i(t)  
& =\frac{ \exp( -\eta(t) L_{i}(t-1) )} { \sum_{j=1}^M \exp ( - \eta(t) L_{j}(t-1) ) }    \notag
\end{align}
where $\{ \eta(t) \}_{t \in \mathbb{N}^+}$ is a positive non-increasing sequence. 
This implies that AH will choose the LLs with smaller cumulative error with
higher probability.

\section{Analysis of the Regret}\label{sec:analysis}

In this section we prove bounds on the regrets given in 
(\ref{eqn:experregret}) and (\ref{eqn:ensebleregret}), when the LLs
use IUP and the EL uses AH as their algorithms. The following theorem
bounds the regret of each LL.

\begin{theorem} \label{thm:CUPregret} \textbf{Regret bounds for LL $i$.}
When LL $i$ uses IUP with the partitioning parameter $m_{i} \in \mathbb{N}^+$,
given $\boldsymbol{X}^T_i = \boldsymbol{x}^T_i$ we have 
\begin{align}
\text{Reg}_{i}(T | \bs{X}_i^T = \boldsymbol{x}^T_i) & \leq 1+ 2 L d^{\alpha/2}_i m^{-\alpha}_i  T +  |{\cal F}_i| m_i^{d_i} \notag \\
&+ 2 A_{m_i}  \sqrt{ |{\cal F}_i| m_i^{d_i} T } . \label{eqn:boundedmemory}
\end{align}
Specifically, when $m_{i}=\lceil T^{1/(2\alpha+d_{i})}\rceil$, we have
\begin{align}
\text{Reg}_{i}(T | \bs{X}_i^T = \boldsymbol{x}^T_i) & \leq T^{\frac{\alpha+d_{i}}{2\alpha+d_{i}}}
C_{i} + T^{\frac{d_{i}}{2\alpha+d_{i}}}  2^{d_i} | {\cal F}_i | + 1  \label{eqn:conditionedregret}
\end{align}
where 
$C_{i} =2 A_{m_i} |{\cal F}_i|^{1/2} 2^{d_i/2} + 2 L d_i^{\alpha/2} $ and
$A_{m_i}  = 2 \sqrt{ 2 (1 + 2 \log ( 2 |{\cal F}_i| m_i^{d_i} T^{\frac{3}{2}} ) )  }$.
From \eqref{eqn:conditionedregret}, it immediately follows that 
\begin{align}
\text{Reg}_{i}(T | \bs{X}_i^T )  & \leq T^{\frac{\alpha+d_{i}}{2\alpha+d_{i}}}
C_{i} + T^{\frac{d_{i}}{2\alpha+d_{i}}}  2^{d_i} | {\cal F}_i | + 1  \text{ a.s.} \notag \\
\overline{\text{Reg}}_{i}(T)  & \leq T^{\frac{\alpha+d_{i}}{2\alpha+d_{i}}}
C_{i} + T^{\frac{d_{i}}{2\alpha+d_{i}}}  2^{d_i} | {\cal F}_i | + 1 . \notag
\end{align}
\end{theorem} 
\begin{proof}
See Appendix \ref{app:theorem1}.
\end{proof}

Theorem \ref{thm:CUPregret} states that the difference
between the expected number of correct predictions made by LO $i$
and IUP increases as a sublinear
function of the sample size $T$. Time order of the terms that appear in \eqref{eqn:boundedmemory} are balanced when $m_{i}=\lceil T^{1/(2\alpha+d_{i})}\rceil$.
This means that the average excess
prediction error of IUP compared to the \rev{optimal policy} converges to zero as the number of data instances grows (approaches
infinity). The regret bound enables us to
exactly calculate how far IUP is from the optimal strategy for any
finite $T$, in terms of the average number of correct predictions.
Basically, we have 
$\overline{\text{Reg}}_{i}(T) /T = \tilde{O} ( T^{ -\frac{\alpha}{2\alpha+d_{i}} } )$. \rev{Moreover the rate of growth of the regret, which is  $\tilde{O} ( T^{\frac{\alpha+d_{i}}{2\alpha+d_{i}}} )$ is optimal \cite{lu2010contextual} (up to a logarithmic factor), i.e., there exists no other learning algorithm that can achieve a smaller rate of growth of the regret.} 

\begin{remark}\label{rem:boundedmemory}
The memory complexity of IUP is $O(  |{\cal F}_i| m^{d_i}_i)$. For $m_{i}=\lceil T^{1/(2\alpha+d_{i})}\rceil$, it becomes $O( |{\cal F}_i| T^{d_i/(2\alpha+d_i)})$. For memory bounded LLs, with a bound $M_i \in \mathbb{N}^+$ on the partitioning parameter, we can set 
$m_{i}= \min \{  \lceil T^{1/(2\alpha+d_{i})}\rceil, M_i \}$. In this case, LL $i$ will incur sublinear regret when $\lceil T^{1/(2\alpha+d_{i})}\rceil \leq M_i$. Otherwise, the regret may not be sublinear. However, we can still obtain an approximation guarantee for IUP, since 
$\lim_{T \rightarrow \infty} \text{Reg}_{i}(T | \bs{X}_i^T ) / T = 2 L d^{\alpha/2}_i m^{-\alpha}_i$. 
This implies that IUP's average reward will be within $2 L d^{\alpha/2}_i M^{-\alpha}_i$ of the average reward of LO $i$. 
\end{remark}

\begin{remark}\label{rem:parametereffect}
Time order of the regret decreases as $\alpha$ increases (\rev{given that $T > d_i^{\alpha+d_i/2}$} holds. Otherwise, the bound given in Theorem \ref{thm:CUPregret} becomes trivial). This can be observed by investigating Assumption \ref{ass:1}. Given two instances $x$ and $x'$ and a prediction rule $f$, as $\alpha$ increases, difference between the prediction accuracies of $f$ for two instances $x$ and $x'$ that lie in the same set of the partition decreases. The constant that multiplies the time order of the regret increases as $L$ increases. This holds because as $L$ increases, the difference between prediction accuracies of $f$ for $x$ and $x'$ may become larger. 
\end{remark}

As a corollary of the above theorem, we have the following {\em
confidence bound} on the accuracy of the predictions of LL $i$ made
by using IUP.

\begin{corollary} \label{cor:confidencebound} \textbf{Confidence
bound for LL $i$.} Assume that LL $i$ uses IUP with the value of the partitioning parameter $m_i$ given in Theorem \ref{thm:CUPregret}. Let $\text{ACC}_{i,\epsilon}(t)$
be the event that the prediction rule chosen by IUP for LL $i$ at
time $t$ has accuracy greater than or equal to $\pi^*_i(x_i(t))-\epsilon$.
For any time $t$, we have
$\mr{Pr}(\text{ACC}_{i,\epsilon_{t}}(t))  \geq 1 -  1/T$,
where 
$\epsilon_t = 
\sqrt{ \frac{8}{N^i_{a_i(t),p_i(t)} (t)} (1 + 2 \log ( 2 |{\cal F}_i| m_i^{d_i} T^{\frac{3}{2}} ) )   }  
+ 2 L d_i^{\alpha/2} T^{\frac{-\alpha}{2\alpha+d_i}} $.
\end{corollary}
\begin{proof}
See Appendix \ref{app:corollary1}.
\end{proof}

Corollary 1 gives a confidence bound on the predictions made by IUP
for each LL. This guarantees that the prediction made by IUP is very
close to the prediction of the best prediction rule that can be selected
given the instance. For instance, in medical informatics, the result
of this corollary can be used to calculate the patient sample size
required to achieve a desired level of confidence in the predictions
of the LLs. For instance, for every $(\epsilon,\delta)$ pair, we
can calculate the minimum number of patients $N^{*}$ such that, for
every new patient $n>N^{*}$, IUP will not choose any prediction rule
that has suboptimality greater than $\epsilon>0$ with probability
at least $1-\delta$, when it exploits (To achieve this we need to set the second term in \eqref{eqn:UCBindex} appropriately). Moreover, Corollary
1 can also be used to determine the number of patients that need to
be enrolled in a clinical trial to achieve a desired level of confidence
on the effectiveness of a drug. 

The theorem below bounds the pseudo-regret of AH for any realization of LLs' rewards, hence almost surely. 

\begin{theorem} \label{thm:Ahedgeregret}
When AH is run with \rev{learning parameter} $\eta(t) = \sqrt{ \log M / t}$, for any reward sequence $\{ \bs{v}^T_i \}_{i \in {\cal M}}$, the pseudo-regret of the EL with respect to the best LL is bounded by
$\text{Reg}_{\text{EL}}(T) \leq 2 \sqrt{T \log M}$.
Hence, we have
$\max_{i \in {\cal M}} \sum_{t=1}^T V_i(t) 
-  \mr{E}\left[ \sum_{t=1}^T  R_{\text{EL}}(t) \right]     
 \leq 2 \sqrt{T \log M} \text{ a.s.}$,
where the expectation is taken with respect to the randomization of the EL.
\end{theorem}
\begin{proof}
See Appendix \ref{app:hedgeproof}.
\end{proof}

The next theorem shows that the expected regret of the EL given in \eqref{eqn:ensebleregret} grows
sublinearly over time and the term with the highest regret
order scales with $F_{\max} = \max_{i \in {\cal M}} |{\cal F}_i|$ rather than $|{\cal F}|$, which is the sum of the number
of prediction rules of all the LLs.

\begin{theorem} \label{thm:regretensemble} \textbf{Regret bound
for the EL} When the EL runs AH with learning parameter
$\eta(t) = \sqrt{ \log M / t  }$ and all LLs run IUP with the partitioning parameter given in Theorem \ref{thm:CUPregret},
the expected regret of the EL with respect to the \rev{best LO $i^*$} is bounded by 
\begin{align*}
\text{Reg}_{\text{EL}}(T) &\leq T^{ \frac{ \alpha+d_{i^*} }{ 2\alpha+d_{i^*}  } } C_{i^*} 
+ T^{ \frac{ d_{i^*} }{ 2\alpha+d_{i^*} } } 2^{d_{i^*}} | {\cal F}_{i^*} | \notag \\
&+ 2 \sqrt{T \log M} + 1 
\end{align*}
where the definition of $C_{i^*}$ is given in Theorem \ref{thm:CUPregret}. 
\end{theorem} 
\begin{proof}
See Appendix \ref{app:theorem2}.
\end{proof}

Theorem \ref{thm:regretensemble} proves that the highest time order
of the regret does not depend on $M$, since $C_{i^*}$ only depends on $|{\cal F}_{i^*}| \leq F_{\max}$ but not on $|\cup_{i \in {\cal M}} {\cal F}_i|$. This implies that the effect of the number of LLs to the learning rate is negligible. 
Since regret is measured
with respect to the optimal data-dependent prediction strategy of the best LL (identical to the best LO),
the benchmark will generally improve as LLs with higher performances
are added to the system. Moreover, the learning loss with respect to the benchmark is only slightly affected by introducing new LLs to the system. Therefore, the performance of the EL will generally improve as LLs with higher performances are added to the system.

\section{Extensions}\label{sec:extensions}

\textbf{Active EL}: Since IUP selects a prediction rule with high uncertainty when it explores, the prediction accuracy of an LL can be low when it explores. 
Since the EL combines the predictions of the LLs, taking
into account the prediction of an LL which explores can reduce the
prediction accuracy of the EL. In order to overcome this limitation, we propose
the following modification: Let ${\cal A}(t)\subset{\cal M}$ be the
set of LLs that exploit at time $t$. If ${\cal A}(t)=\emptyset$,
the EL will randomly choose one of the LLs' prediction as its final
prediction. Otherwise, the EL will apply an ensemble learning algorithm
(such as AH or WM) using only the LLs in ${\cal A}(t)$. This means
that only the predictions of the LLs in ${\cal A}(t)$ will be used
by the EL and only the weights of the LLs in ${\cal A}(t)$ will be
updated by the EL. Our numerical results illustrate that such a modification
can indeed result in an accuracy that is much higher than the accuracy of
the best LL.

\textbf{Contextual EL (CEL):} The predictive accuracy of the EL can be further improved if it can observe a set of contexts that yields additional information about the accuracies of LLs' prediction rules. 
For instance, these contexts can be a subset of the data instances that LLs observe, or some other side observation about the instance that the EL currently examines.

We assume that CEL can observe $d_{\text{EL}}$-dimensional context in addition to the predictions of the LLs. Let $x_\text{EL}(t)$ be the context observed at time $t$ by the EL, which is an element of ${\cal X}_{\text{EL}} = [0,1]^{d_{\text{EL}}}$. The learning algorithm we propose for CEL is called {\em Contextual Hedge} (CH).
Similar to IUP, CH partitions the context space into equal sized, identically shaped, non-overlapping sets, and learns a different LL selection rule for each set in the partition. With this modification, the EL can learn the best LL for each set in the partition, which will yield a higher predictive accuracy than learning the best LL only based on the number of correct predictions. 

The pseudocode of the CH is given in Fig. \ref{fig:FBEL}. CH runs a different instance of the AH in each set $p$ of its context space partition ${\cal P}_{ \text{EL} }$. The cumulative loss vector it keeps for $p$ at time $t$ is denoted by $\bs{L}_p(t) = (L_{p,1}(t), \ldots, L_{p,M}(t))$, where $L_{p,i}(t)$ denotes the number of prediction errors made by LL $i$ by the end of time step $t$ for contexts that arrived to $p$. \rev{$N_{\text{EL}, p}(t)$  denotes the number of context arrivals to $p$ by the end of time $t$.}
At the beginning of time step $t$, CH identifies the set in ${\cal P}$ that $x_\text{EL}(t)$ belongs to, which is denoted by $p_\text{EL}(t)$. After CH receives the set of predictions $\hat{\bs{h}}(t)$ of the LLs, it samples its final prediction from this set according to probability distribution $\bs{q}(t) = (q_1(t), \ldots, q_M(t))$, where 
\begin{align}
& \mr{Pr}(\hat{Y}(t)=\hat{h}_{i}(t)) = q_{i}(t)  \notag \\
& =\frac{ \exp( -\eta( N_{\text{EL}, p_\text{EL}(t)}(t) ) L_{p_\text{EL}(t),i}( t - 1      ) )} 
{ \sum_{j=1}^M \exp ( - \eta( N_{\text{EL}, p_\text{EL}(t)}(t) ) L_{p_\text{EL}(t),j}( t -1  ) ) } .   \notag
\end{align}
Standard Hedge algorithm is not suitable in this setting because it requires the knowledge of $N_{\text{EL}, p}(T)$ beforehand for each $p \in {\cal P}_{ \text{EL} }$. However, AH works properly because it can update its learning parameter $\eta(\cdot)$ on-the-fly for each $p \in {\cal P}_{ \text{EL} }$, using the most recent value of $N_{\text{EL}, p}(t)$. 
Let ${\cal Z}_p(t) := \{ l \leq t: x_{\text{EL}}(l) \in p \}$ denote the set of times in which the context is in $p$ by time $t$. 
For a given sequence of LL rewards $\{ \bs{v}^T_i \}_{i \in {\cal M}}$ and context arrivals $\bs{x}^T_{\text{EL}}$
we define the best LL for set $p \in {\cal P}_{\text{EL}}$ of the EL as
\begin{align}
i^*_p \in \argmax_{i \in {\cal M}} \sum_{l \in {\cal Z}_p(T)}  v_i(l) .     \notag
\end{align}
The \rev{contextual pseudo-regret} of CEL is defined as 
\begin{align}
\text{Reg}_{\text{CEL}}(T) := \sum_{t=1}^T v_{ i^*_{ p_{\text{EL}}(t) } }(t) 
- \mr{E} \left[ \sum_{t=1}^T R_{\text{EL} }(t)\right]  \label{eqn:newregret}
\end{align}
where the expectation is taken with respect to the randomization of CH.
The following theorem bounds the regret of CH based on the granularity of the partition it creates.

\begin{theorem} \label{thm:CHregret}
\textbf{Regret bound
for CH.} When CEL runs CH with learning parameter
$\eta(t) = \sqrt{ \log M / t}$ and partitioning parameter $m_\text{EL}$,
the contextual pseudo-regret of the CEL is bounded by
\begin{align}
\text{Reg}_{\text{CEL}}(T) \leq  2 \sqrt{ T ( m_{\text{EL}} )^{d_{\text{EL}}} \log M  }  \notag
\end{align}
for any $( \{ \bs{v}^T_i \}_{i \in {\cal M}}, \bs{x}^T_{\text{EL}})$.
\end{theorem} 
\begin{proof}
See Appendix \ref{app:CHregret}.
\end{proof}
The regret bound given in Theorem \ref{thm:CHregret} is obtained without making any distributional assumptions on data instance and context arrivals. Given a fixed time horizon $T$, this regret bound increases at rate $m_{\text{EL}}^{d_{\text{EL}}/2}$. Since the trivial regret bound $\text{Reg}_{\text{CEL}}(T) \leq T$ always holds, the bound in Theorem \ref{thm:CHregret} guarantees that the regret is sublinear only if $m_{\text{EL}} < (T / ( 4 \log M ))^{1/d_{\text{EL}}}$.
It might seem counter-intuitive that the regret is minimized when $m_{\text{EL}}=1$. The reason for this is that our benchmark $\sum_{t=1}^T v_{ i^*_{ p_{\text{EL}}(t) } }(t) $ given in the left-hand side of \eqref{eqn:newregret} reduces to the benchmark $\max_{i \in {\cal M}} \sum_{t=1}^T v_i(t)$ given in \eqref{eqn:pseudoregret} when $m_{\text{EL}}=1$. The next lemma shows that the reward of the benchmark in \eqref{eqn:newregret} is non-decreasing in $m_{\text{EL}}$ when $m_{\text{EL}}$ is chosen from $\{1,2,4,8,\ldots\}$.

\begin{lemma} \label{lemma:improvedbenchmark}
Consider $m'$ and $m$ in $\{1,2,4,8,\ldots\}$ such that $m'> m$. Let ${\cal P}'$ (${\cal P}$) be the partition of ${\cal X}_{\text{EL}}$ formed by $m'$ ($m$). Let $p'(t)$ ($p(t)$) denote the set in ${\cal P}'$ (${\cal P}$) that $x_{\text{EL}}(t)$ belongs to. 
For any $( \{ \bs{v}^T_i \}_{i \in {\cal M}}, \bs{x}^T_{\text{EL}})$, we have 
\begin{align}
\sum_{t=1}^T v_{ i^*_{ p'(t) } }(t) \geq  \sum_{t=1}^T v_{ i^*_{ p(t) } } (t)  .       \notag
\end{align}
\end{lemma}
\begin{proof}
Due to the fact that $m'$ and $m$ are chosen from $\{1,2,4,8,\ldots\}$, each $p' \in {\cal P}'$ is included in exactly one $p \in {\cal P}$\footnote{Assignment of the contexts that lie on the boundary to one of the adjacent sets can be done in any predetermined way without affecting the result.} Moreover, each $p \in {\cal P}$ includes exactly $(m'/m)^{d_{\text{EL}}}$ sets in ${\cal P}'$. Let ${\cal S}_p$ denote the set of $p' \in {\cal P}'$ such that $p' \subset p$.
For any $p \in {\cal P}$ we have
\begin{align}
\max_{ i \in {\cal M} } \sum_{l \in {\cal Z}_p(T)} v_i(l) \leq \sum_{ p' \in {\cal S}_p } \max_{i \in {\cal M}} \sum_{l \in {\cal Z}_{p'}(T)} v_i(l) .    \notag
\end{align} 
Hence,
\begin{align}
\sum_{t=1}^T v_{ i^*_{ p(t) } } (t) &=
 \sum_{p \in {\cal P}} \max_{i \in {\cal M}} \sum_{l \in {\cal Z}_p(T) } v_i(l)      \notag \\
 & \leq  \sum_{p \in {\cal P}} \sum_{ p' \in {\cal S}_p } \max_{i \in {\cal M}} \sum_{l \in {\cal Z}_{p'}(T)} v_i(l) \notag \\
& =  \sum_{p' \in {\cal P}'} \max_{i \in {\cal M}} \sum_{l \in {\cal Z}_{p'}(T) } v_i(l) 
=  \sum_{t=1}^T v_{ i^*_{ p'(t) } }(t) . \notag 
\end{align}
\end{proof}
Theorem \ref{thm:CHregret} and Lemma \ref{lemma:improvedbenchmark} shows the tradeoff between approximation and estimation errors. The benchmark we compare CH against improves (never gets worse) as $m_{\text{EL}}$ increases. Ideally, we would like CH to compete with $\sum_{t=1}^T v_{ i^*_{ \{ x_{\text{EL} (t) } \} }} (t)$, i.e., with respect to the best LL given context $x_{\text{EL}} (t)$. For $x_{\text{EL}} (t) \in p$, CH approximates $i^*_{ \{ x_{\text{EL}} (t) \} }$ with $i^*_p$. Learning (estimating) $i^*_{ \{ x_{\text{EL}} (t) \} }$ is harder than learning $i^*_p$ because the past observations that CH can use to learn $i^*_{ \{ x_{\text{EL}} (t) \} }$ is less than or equal to (usually less than) that it can use to learn $i^*_p$.  This is the reason why the regret increases with $m_{\text{EL}}$.  
The optimal value for $m_{\text{EL}}$ can be found by pre-training CH before its online deployment.

\begin{figure}[htb]
\fbox{%
\begin{minipage}[c]{0.95\columnwidth}%
{\fontsize{9}{9}\selectfont \flushleft{Contextual Hedge (CH)} 
\begin{algorithmic}
\STATE{Input: A non-increasing sequence of positive real numbers $\{ \eta(t) \}_{t \in \mathbb{N}^+}$, $m_{ \text{EL} }$ and $d_{ \text{EL} }$}
\STATE{Initialize sets: Create partition ${\cal P}_{ \text{EL} }$ of $[0,1]^{d_{\text{EL}}}$ into $(m_{\text{EL}})^{d_{\text{EL}}}$ identical hypercubes}
\STATE{Initialize counters: $N_{\text{EL}, p} = 0, \forall p \in {\cal P}_\text{EL}$, $t=1$} 
\STATE{Initialize losses: $L_{i,p} = 0$, $\forall i \in {\cal M}$, 
$p \in {\cal P}_{ \text{EL} }$ } 
\WHILE{$t\geq1$} 
\STATE{Find the set in ${\cal P}_{ \text{EL} }$ that $x_\text{EL}(t)$ belongs to, i.e., $p_\text{EL}(t)$. Let $p = p_\text{EL}(t)$ }
\STATE{Set $N_{\text{EL}, p} \leftarrow N_{\text{EL}, p} +1 $}
\STATE{Receive predictions of LLs: $\hat{\bs{h}}(t)$}
\STATE{Choose the LL $I(t)$ to follow according to the distribution $\bs{q}(t) := (q_1(t),\ldots, q_M(t))$ where
\begin{align}
q_{i}(t) = \frac{ \exp( -\eta(N_{\text{EL}, p}) L_{i,p} )} { \sum_{j=1}^M \exp ( - \eta(N_{\text{EL}, p}) L_{j,p} ) }      \notag
\end{align}
}
\STATE{Predict $\hat{y}(t) = \hat{h}_{I(t)}(t)$}
\STATE{Observe the true label $y(t)$}
\STATE{Receive the reward $r_{\text{EL}}(t) = \mr{I} (\hat{y}(t) = y(t))$ and observe losses of all LLs: $l_i(t) = \mr{I} (\hat{h}_i(t) \neq y(t))$ for $i \in {\cal M}$}
\STATE{Set $L_{i,p} \leftarrow  L_{i,p} + l_i(t)$ for $i \in {\cal M}$}
\STATE{$t \leftarrow t+1$} 
\ENDWHILE
\end{algorithmic} 
} %
\end{minipage}}\protect\protect\caption{Pseudocode of CH.}
\label{fig:FBEL} 
\end{figure}
\com{$N_{\text{EL}, p}$ is updated after it is used to estimate the probability. Hence, the value of $N_{\text{EL}, p}$ at the time the probability is calculated is the number of context arrivals to hypercube $p$ of the EL prior to time $t$.}

\section{Illustrative Results} \label{sec:numerical}

In this section, we evaluate the performance of several HB-based methods and compare them with numerous other state-of-the-art machine learning methods on a breast cancer diagnosis dataset from the UCI archive \cite{mangasarian1995breast}.

\vspace{-0.1in}
\subsection{Simulation Setup} \label{sec:simusetup}

\textbf{Description of the dataset:} The original dataset contains 569 instances and 32 attributes, of which one attribute is the ID number of the patient and one attribute is the label. Each instance contains features extracted from the images of fine needle aspirate (FNA) of breast mass.
There are 30 clinically relevant attributes. The diagnosis
outcome (label) is whether the tumor of the patient is malignant or benign.

\textbf{Benchmarks:} We compare HB with several state-of-the-art centralized and decentralized benchmarks. A centralized benchmark is a machine learning algorithm that has access to all the features of an instance. A decentralized benchmark on the other hand, applies the same LL and EL structure as the HB. Hence, each LL has access to a subset of features. However,
the algorithms used to train the LL and the EL are different from
the HB.

In the first set of experiments, we compare the HB methods with centralized benchmarks such as Support Vector Machine (SVM) and Logistic Regression (LR). In the second set of experiments, we study the performance of various ensemble learning methods for the EL, by fixing the learning algorithm of the LLs as IUP. In the third set of experiments, we evaluate the impact of system variables such as the number of LLs and past history on the performance of the HB methods. In the fourth set of experiments, we consider the extensions described in Section \ref{sec:extensions}.

The list of the algorithms used by the
EL in this section is given below.
\begin{itemize}
\item Adaptive Boosting (AdaBoost) \cite{freund1995desicion}.
\item Perceptron Weighted Majority (PWM) \cite{canzian2015timely, canzian2013ensemble}.
\item Blum's variant of Weighted Majority (Blum) \cite{blum1997empirical}.
\item Herbster's variant of Weighted Majority (TrackExp) \cite{herbster1998tracking}.
\end{itemize}
We also compare performance of the HB with standard benchmarks that are widely used in learning theory, which are listed below.
\begin{itemize}
\item Best LL: LL with the highest accuracy over the dataset. 
\item Worst LL: LL with the lowest accuracy over the dataset. 
\item Average LL: Accuracy averaged over all the LLs. 
\end{itemize}

When IUP is used, we assume that each LL has two prediction rules: rule $1$ always predicts malignant, and rule $2$ always predicts benign. Hence, using IUP, each LL is learning the best prediction for each set in its feature space partition.

\textbf{General setup:} For all the simulations, each algorithm is run $50$ times. The reported results correspond to the averages taken over these runs.

For the HB, we create $3$ LLs, and randomly assign $10$ attributes to each LL as its feature types for each run independently. The LLs do not have any common attributes.
Hence, $d_{i}=10$ for all $i\in\{1,2,3\}$. Each run of the HB is
done over $T = 10000$ data instances that are drawn independently and
uniformly at random from the $569$ instances of the original dataset except Experiment 1 and 2, in which training and test samples are separated (for offline algorithms).

\textbf{Performance metrics:} We report three performance metrics
for the above experiments: prediction error rate (PER), false positive rate (FPR) and false negative rate (FNR). PER is defined as the fraction of times the prediction is different from the true label. FPR and FNR are defined as the prediction error rate for benign cases and malignant cases, respectively. The main goal of diagnosis is to minimize the FPR, given a tolerable threshold for the FNR selected by the system
user. In the simulations, the threshold for FNR is set to be $3\%$, which is considered to be a reasonable level in breast cancer diagnosis \cite{linqi2015health}. Using this threshold, we can re-characterize the performance metric as follows. 
\begin{align}
\text{minimize FPR subject to FNR} \leq 3\% . \notag
\end{align}

FNR can be set below $3\%$ by  introducing a {\em hyper-parameter} which trade-offs FPR and FNR. The details are explained below. 

For IUP, $\hat{\pi}^i_{1,p}$ ($\hat{\pi}^i_{2,p}$) denotes the estimated accuracy for malignant (benign) classifier for feature set $p$ of LL $i$. \nrev{Prediction is performed using the indices given in \eqref{eqn:UCBindex}.}
LL $i$ will predict malignant if \nrev{$g^i_{1,p} \geq g^i_{2,p}$}.\footnote{Without loss of generality, we assume that the prediction of LL $i$ is malignant when \nrev{$g^i_{1,p} = g^i_{2,p}$}.} Otherwise, it will predict benign. 
Let $h_{\text{IUP}}$ be the hyper-parameter for IUP. We can modify the prediction rule of IUP as follows: LL $i$ predicts malignant if \nrev{$h_{\text{IUP}} \times g^i_{1,p} \geq g^i_{2,p}$}. Otherwise, it predicts benign. It is obvious that when $h_{\text{IUP}} > 1$, LL $i$ classifies more cases as malignant, which yields a decrease FNR and an increase FPR.

For SVMs and logistic regression, the hyper-parameter is the decision boundary between the malignant and benign cases. Assume that we assign label $1$ to the malignant case and $0$ to the benign case. An unbiased decision boundary will classify every output that is greater than $0.5$ as malignant and less than $0.5$ as benign. If we perturb the decision boundary such that it lies below $0.5$, then it is expected that SVM and LR classify more cases as malignant. This yields a decrease in FNR and an increase in FPR. 

In order to set FNR just below $3\%$, we first randomly select a hyper-parameter value and run the corresponding algorithm 50 times, and then calculate FPR and FNR. After this step, the hyper-parameter is adjusted to minimize the distance between FNR and the threshold. The reported PER and FPR correspond to the ones that are obtained for the hyper-parameter value which makes FNR just below $3\%$.

To compare the performance of various algorithms, we introduce the concept of {\em improvement ratio} (IR). Let $\text{PM(A)}$ denote the performance of algorithm A for metric $\text{PM}$. $\text{PM}$ can be any loss metric such as PER, FPR, FNR. The IR of algorithm A with respect to algorithm B is defined as 
\begin{align}
( \text{PM(B)} - \text{PM(A)} )/ \text{PM(B)}   .  \notag
\end{align}

\vspace{-0.1in}
\subsection{Experiment 1 (Table \ref{table:comprein}, Fig. \ref{fig:traininga})}
This experiment compares HB against LR, SVM,
AdaBoost (all trained offline); and Best LL, Average LL and Worst
LL benchmarks. The training of the offline methods is performed as follows. LR and SVM are trained in a centralized way and have access to all 30 features. In the test phase, they observe all the 30 features of the new instance and make a prediction. AdaBoost is trained in a decentralized way. It has $3$ weak learners (logistic regression with different parameters), which are randomly assigned to 10 of the 30 attributes.
These weak learners do not share any common attributes.

For each run, offline methods are trained using different 285 ($50\%$) randomly drawn instances from the original 569 instances. Then, the performances of both the HB and benchmarks are evaluated on $10,000$ instances drawn uniformly at random from the remaining $284$ instances (excluding
$285$ training instances) for each run. 

As Table \ref{table:comprein} shows, HB (IUP + WM) has $2.96\%$
PER and $2.61\%$ FPR when the FNR is set to be just below
$3\%$. Hence, the PER IR of HB (IUP + WM) with respect to the best benchmark algorithm (LR) is 0.51. We also note that the PER IR of the best LL with respect to the second best algorithm is 0.44. This implies that the IUP used by the LLs yields high classification accuracy, because it is able to learn online the best prediction given the types of features seen by each LL.

HB with WM outperforms the best LL, because it takes a weighted majority of the predictions of LLs as its final prediction, rather than relying on the predictions of a single LL. As observed from Table \ref{table:comprein}, all LLs have reasonably high accuracy, since PER of the worst LL is $6.23\%$. In contrast to WM, AH puts a probability distribution over the LLs based on their weights, and follows the prediction of the chosen LL. With highly accurate LLs, the deterministic approach (WM) works better than the probabilistic approach (AH), because in almost all time steps, the majority of
the LLs make correct predictions.

\begin{table}[h]
\caption{Comparison of HB with offline benchmarks}
\label{table:comprein}
\begin{center}
\begin{tabular}{| c | c | c | c | c | c | c |}
\hline 
{Units(\%)}& \multicolumn{3}{|c|}{{\bf Average}}   & \multicolumn{3}{|c|}{{\bf Standard Deviation}} \\ 
\hline
Performance Metric& {\bf PER} & {\bf FPR} & {\bf FNR} & {\bf PER}& {\bf FPR} & {\bf FNR} \\
\hline 
{\bf HB(IUP+WM)} & 2.96 & 2.61 & 2.99 & 0.73 & 1.26 & 0.88\\
\hline 
{\bf HB(IUP + AH)} & 3.83 & 4.46 & 2.98 & 0.85 & 1.43 & 0.79 \\
\hline 
{\bf Logistic Regression} & 6.04 & 8.48 & 2.94 & 2.18 & 4.07 & 1.3 \\
\hline 
{\bf AdaBoost} & 6.91 & 9.55 & 2.99 & 2.58 & 4.82 & 1.83 \\
\hline 
{\bf SVMs} & 9.73 & 14.21 & 2.98 & 2.5 & 4.19 & 1.98 \\
\hline 
{\bf Best LL of IUP} & 3.39 & 3.57 & 2.99 & 0.85 & 1.43 & 0.79 \\
\hline 
{\bf Average LL of IUP} & 4.68 & 6.17 & 2.97 & 0.89 & 1.49 & 0.85 \\
\hline 
{\bf Worst LL of IUP} & 6.23 & 9.06 & 2.99 & 1.62 & 2.71 & 1.54 \\
\hline 
\end{tabular}
\end{center}
\end{table}

\begin{figure}
\includegraphics[width=0.9\columnwidth]{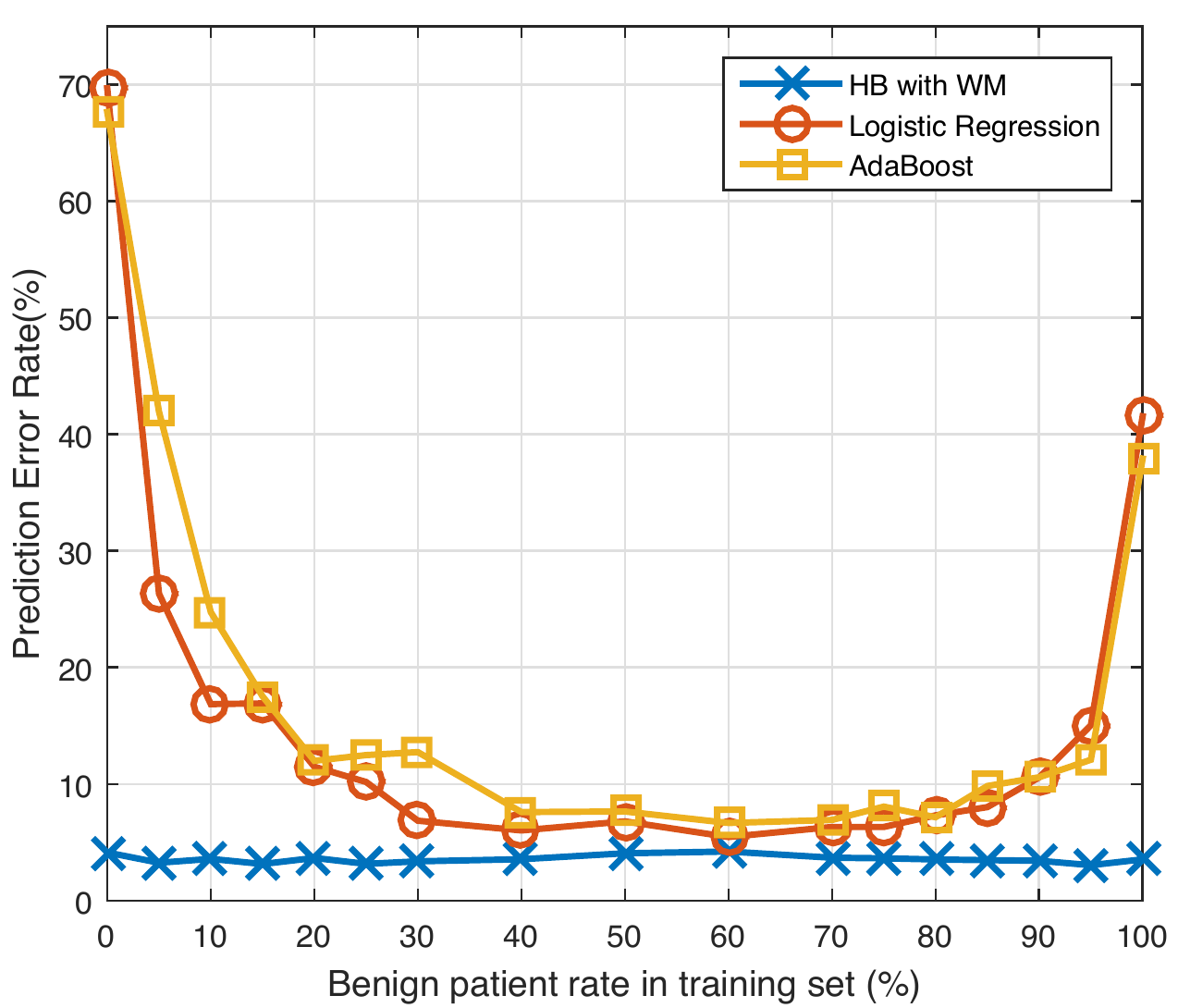}
\caption{PER of HB, LR and AdaBoost as a function of the composition of the training set.}
\label{fig:traininga}
\end{figure}

Another advantage of HB is that it has low standard deviation for PER, FPR and FNR, which is expected since IUP provides tight confidence bounds on the accuracy of the prediction rule chosen for any instance for which it exploits.

In Fig. \ref{fig:traininga}, the performances of HB (with WM), LR and AdaBoost are compared as a function of the training set composition. Since both LLs and the EL learn online in HB, its performance does not depend on the training set composition. On the other hand, the performance of LR and AdaBoost highly depends on the composition of the training set.
Although these benchmarks can be turned into an online algorithm by retraining them after every time step,
the computational complexity of the online implementations for these algorithms will be high compared to that of the HB.
Therefore, implementing the online versions of these benchmarks are not feasible when the dataset under consideration is large, and decisions have to be made on-the-fly.

\vspace{-0.1in}
\subsection{Experiment 2 (Table \ref{table:secondtable})}
  This experiment compares HB against four ensemble
learning algorithms: AdaBoost, PWM, Blum and TrackExp. The goal of this experiment is to assess how the algorithm used by the EL impacts the performance. To isolate this effect, all the LLs use the same learning algorithm. The learning algorithms we use for the LLs are IUP (online), LR and SVM (offline). In this experiment, the performance metric is the accuracy for the 1001st patient. All of the other simulation details are exactly the same as in Experiment 1.
\begin{table*}[t]
\caption{Comparison of HB with other ensemble learning methods.}
\label{table:secondtable}
\begin{center}
\begin{tabular}{| c | c | c | c | c | c | c | c | c | c |}
\hline 
Unit(\%)& \multicolumn{9}{|c|}{{\bf Local Learner Algorithm}} \\ 
\hline
& \multicolumn{3}{|c|}{{\bf IUP}} & \multicolumn{3}{|c|}{{\bf Logistic Regression}} & \multicolumn{3}{|c|}{{\bf SVMs}}\\ 
\hline
Ensemble Learning Algorithm& {\bf PER} & {\bf FPR} & {\bf FNR} & {\bf PER}& {\bf FPR} & {\bf FNR}& {\bf PER}& {\bf FPR} & {\bf FNR} \\
\hline 
{\bf HB(with WM)} & \textbf{2.72} & \textbf{1.96} & 2.94 & \textbf{2.81} & \textbf{2.71} & 2.97 & \textbf{3.43} & \textbf{3.51} & 2.97 \\
\hline 
{\bf HB(with AH)} & 4.35 & 5.05 & \textbf{2.93} & 3.61 & 3.81 & 2.95 & 4.63 & 5.11 & 2.97\\
\hline 
{\bf AdaBoost} & 3.02 & 3.09 & 2.98 & 3.28 & 3.15 & 2.97 & 4.27 & 4.16 & 2.95 \\
\hline 
{\bf PWM} & 3.11 & 2.82 & 2.96 & 2.96 & 3.08 & 2.96 & 3.95 & 4.55 & 2.96 \\
\hline 
{\bf Blum} & 3.09 & 3.12 & 2.93 & 3.5 & 3.78 & 3.00 & 3.68 & 4.18 & 2.95 \\
\hline 
{\bf TrackExp} &2.97 & 2.21 & 2.99 & 3.03 & 3.05 & \textbf{2.93} & 4.02 & 3.81 & 2.99 \\
\hline 
{\bf Best LL} &3.96 & 4.69 & 2.96& 3.22 & 3.3 & 2.99 & 3.96 & 4.26 & 2.96 \\
\hline 
{\bf Average LL} & 5.22 & 7.04 & 2.98 & 4.5 & 4.58 & 2.94 & 5.64 & 5.8 & \textbf{2.92}  \\
\hline 
{\bf Worst LL} & 6.55 & 9.41 & 2.97 & 6.4 & 6.48 & 2.97 & 6.36 & 7.03 & 2.94 \\
\hline 

\end{tabular}
\end{center}
\end{table*}

As seen in Table \ref{table:secondtable},
performance of the HB is better than the other ensemble learning
methods when the FNR threshold is set to $3\%$. More specifically, the performance improvement ratio of HB (with WM) in comparison with the second best algorithm (TrackExp) is 0.08 and 0.11 in terms of PER and FPR when IUP is used by the LLs.
%

\vspace{-0.1in}
\subsection{Experiment 3 (Fig. \ref{fig:fig23})}

This experiment analyzes the performance of the HB as a function of two system parameters: the number of LLs and the dataset size. Firstly, we analyze the performance using different numbers of LLs - from 2 to 30 -, over $10000$ patients (as in Experiment 1). In this simulation, all the LLs have access to different types of attributes. Hence, as the number of LLs increase, the number of attributes per LL decreases. This can be viewed as increasing the amount of decentralization in the system. Secondly, we analyze the performance as a function of the total number of patients
that have arrived so far. For this case, the number of LLs is fixed to $3$.

\textbf{Effect of the number of LLs:} The left Fig. \ref{fig:fig23} shows the performance of the HB with WM and AH as a function
of the number of LLs. In this case, the number of features seen by each LL is roughly equal to $30/M$. As $M$ increases both the performance of the LLs and the EL decreases. The decrease in the performance of the LLs is due to the fact that they see less features, and each LL has less information about the data. The decrease in the performance of the EL is due to the decrease in the performance of the LLs.

\textbf{Effect of the number of previously diagnosed patients:} The right Fig. \ref{fig:fig23} shows the performance of the HB as a function of the number past patients. As expected, the performance improves monotonically with the number of past patients, which is consistent with the regret results we have obtained.

\begin{figure}
\includegraphics[width=\columnwidth]{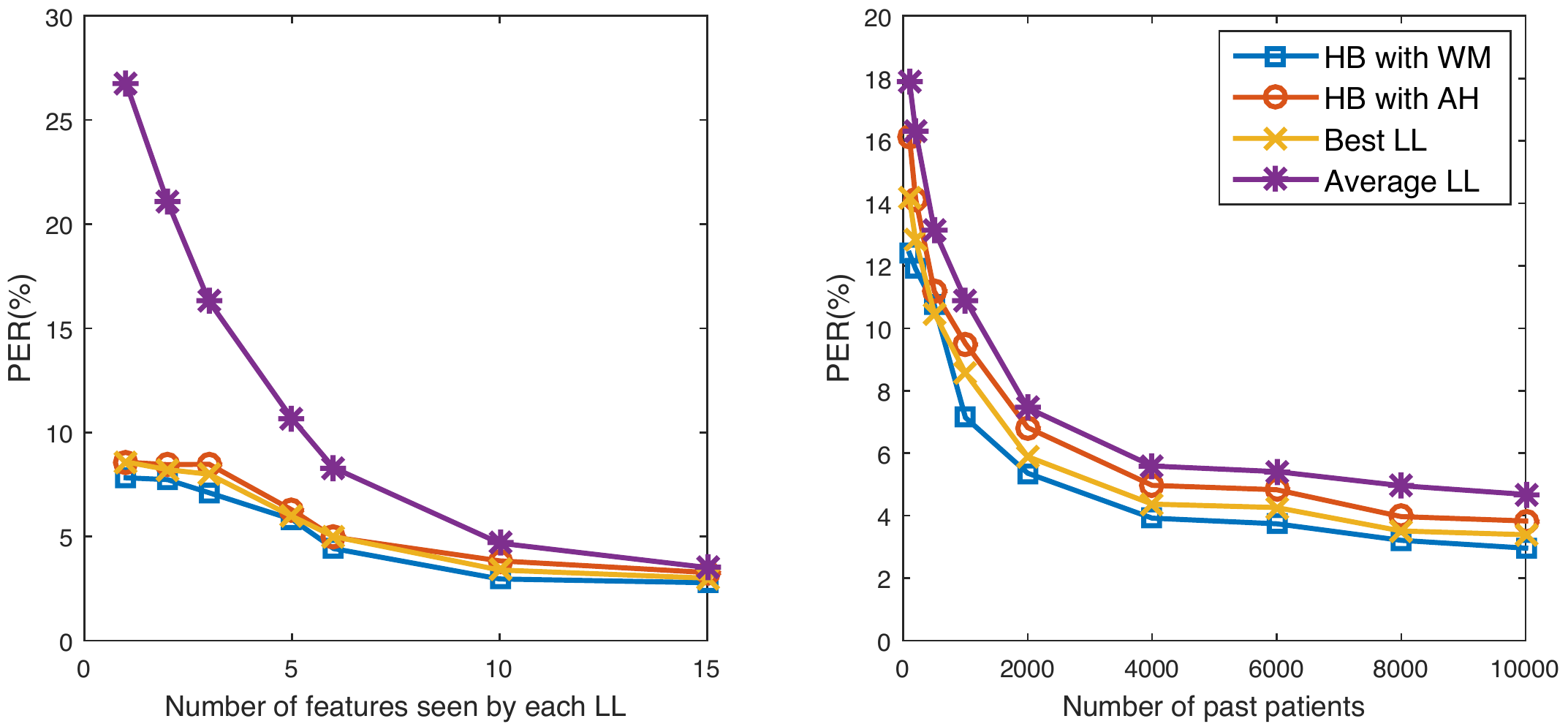}
\caption{{\bf Left:} Number of features seen by each LL vs PER, {\bf Right:} Number of past patients vs PER}
\label{fig:fig23}
\end{figure}

\vspace{-0.1in}
\subsection{Experiment 4}
\subsubsection{Extension 1: Active EL (Table \ref{table:ActiveEL1}, Table \ref{table:ActiveEL2})}

Table \ref{table:ActiveEL1} shows the percentage of times the LLs explore and exploit. 
The LLs are in exploration in 1.5\% of the time steps, and the LLs' overall accuracy in these steps is around 50\%. 

\begin{table}[h!]
	\caption{Performance of exploration step and exploitation step.}
	\label{table:ActiveEL1}
	\begin{center}
		\begin{tabular}{| c | c | c | c |}			
			\hline 
			Units(\%) & Exploration & Exploitation & Average \\
			\hline 
			PER & 50.34 & 3.79 & 4.51 \\
			\hline 
			FPR & 42.88 & 2.51 & 2.94 \\
			\hline
			FNR & 55.88 & 5.93 & 7.11 \\
			\hline			
			Ratio & 1.52 & 98.48 & 100.00 \\
			\hline
		\end{tabular}
	\end{center}
\end{table}

If the EL only considers the predictions of the LLs that exploit (Active EL), both HB with WM and AH have improved performance compared to the original HB, as shown in Table \ref{table:ActiveEL2}.
Specifically, the PER IRs of Active EL (with WM or AH) with respect to the original HB (with WM or AH) are 0.12 and 0.14, respectively.

\begin{table}[h!]
	\caption{Performance improvement with Active EL.}
	\label{table:ActiveEL2}
	\begin{center}
		\begin{tabular}{| c | c | c | c | c |}
			\hline
			Units(\%) & \multicolumn{3}{|c|}{{\bf HB (with WM)}}    \\ 
			\hline
			& Best LL & HB & HB with Active EL \\
			\hline 
			PER & 3.39 & 2.96 & 2.60  \\
			\hline 
			FPR & 3.57 & 2.61 & 2.12  \\
			\hline
			FNR & 2.99 & 2.99 & 2.98  \\
			\hline
			Units(\%) & \multicolumn{3}{|c|}{{\bf HB (with AH)}}    \\ 
			\hline
			& Best LL & HB & HB with Active EL \\
			\hline 
			PER & 3.39 & 3.83 & 3.30  \\
			\hline 
			FPR & 3.57 & 4.46 & 3.46  \\
			\hline
			FNR & 2.99 & 2.98 & 2.98  \\
			\hline
		\end{tabular}
	\end{center}
\end{table}

\subsubsection{Extension 2: Missing and erroneous labels (Fig \ref{fig:fig45})}

In this section, we illustrate the degradation in performance that results from randomly introducing missing or erroneous labels. 
When the label is missing, the LLs and the EL do not update their learning algorithms. 
Fig. \ref{fig:fig45} shows the affect of the missing label rate to the PER. 
It is observed that when 50\% of the labels are missing, the PER degradation is only around 1\% for both HB (with AH) and HB (with WM). This shows the robustness of HB to missing labels. 

\begin{figure}[h!]
	\includegraphics[width=\columnwidth]{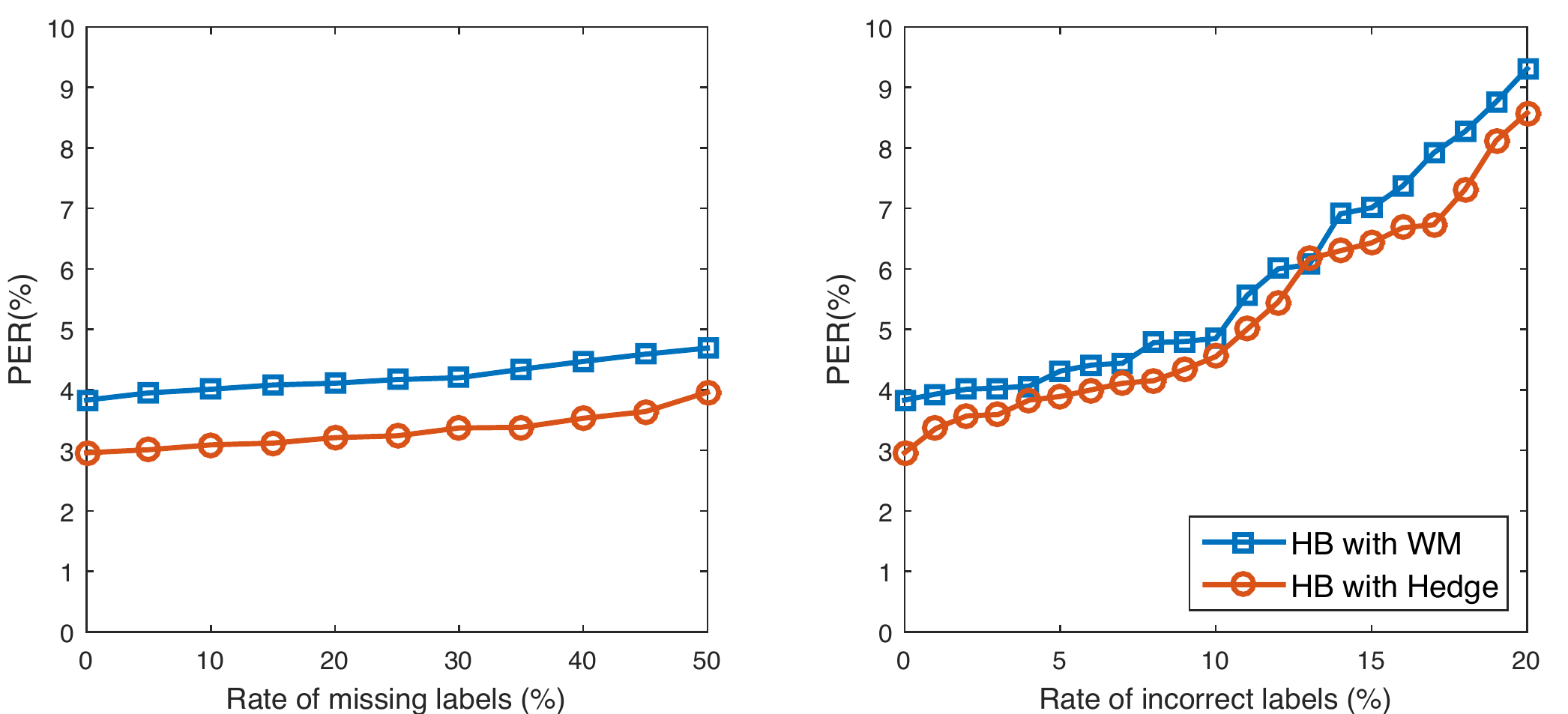}
	\caption{{\bf Left:} Performance degradation due to missing labels. {\bf Right:} Performance degradation due to erroneous labels.}
	\label{fig:fig45}
\end{figure}

Next, we introduce erroneous labels (for binary labels, this correspond to flipped labels). Since the LLs and the EL update their learning algorithms when the label is incorrect, this results inaccurate accuracy estimates. 
Fig. \ref{fig:fig45} shows the affect of the erroneous label rate to the PER. 
For instance, when 10\% of the labels are erroneous, the PER degradation is less than 2\% for both HB (with AH) and HB (with WM).

\subsubsection{Extension 3: Contextual EL (CEL)}
 This experiment studies CEL introduced in Section \ref{sec:extensions}. CEL is compared with the original HB and the best LL (each LL uses IUP). In addition to this, the predictive accuracy of the proposed method as a function of the number of features assigned to the CEL ($d_{\text{EL}}$) is computed. 
The simulation parameters are exactly the same as the parameters used in Experiment 1.

As Table \ref{table:contEL1} shows, CEL with WM has $2.31\%$ PER, $1.48\%$ FPR, and $2.98\%$ FNR. Hence, the performance improvement ratios with respect to the original HB (IUP+WM) approach are 0.22 and 0.43 in terms of PER and FPR, respectively. In addition, the performance IRs with respect to the best LL are 0.32 and 0.59 in terms of PER and FPR, respectively. In other words, CEL with WM significantly outperforms the original HB (IUP+WM) and the best LL in terms of both PER and FPR.
The reason for these improvements is that CEL learns the best LL for each feature set in its partition, rather than learning the best LL in overall.

\begin{table}[h!]
	\caption{Comparison of CEL with original HB and the best LL in terms of PER, FPR and FNR ($d_{EL}=3$ for CEL(WM), $d_{EL}=4$ for CEL(AH))}
	\label{table:contEL1}
	\begin{center}
		\begin{tabular}{| c | c | c | c |}
			\hline 
			Units (\%) & {\bf PER} & {\bf FPR} & {\bf FNR} \\
			\hline 
			{\bf CEL(WM)} & 2.31 & 1.48 & 2.98 \\
			\hline 
			{\bf CEL(AH)} & 3.49 & 4.01 & 2.95 \\
			\hline 
			{\bf HB(WM)} & 2.96 & 2.61 & 2.99 \\
			\hline 
			{\bf HB(AH)} & 3.83 & 4.46 & 2.98  \\
			\hline 
			{\bf Best LL of IUP} & 3.39 & 3.57 & 2.99 \\
			\hline 
		\end{tabular}
	\end{center}
\end{table}

Table \ref{table:contEL2} shows the performance of CEL as a function of the number of features observed by the EL. When WM is used, the performance improves until $d_{\text{EL}}=3$, while when AH is used the performance improves until $d_{\text{EL}}=4$.
The reason that the performance does not improve monotonically with $d_{\text{EL}}$ is the tradeoff between estimation and approximation errors, which is described in detail in Section \ref{sec:extensions}.

 \begin{table}[h!]
\caption{Performance of CEL as a function of the number of features that CEL can observe}
\label{table:contEL2}
\begin{center}
\begin{tabular}{| c | c | c | c | c | c | c |}

\hline 
& \multicolumn{3}{|c|}{{\bf CEL}}   & \multicolumn{3}{|c|}{{\bf CEL}}  \\
                    &  \multicolumn{3}{|c|}{{\bf +WM}}   & \multicolumn{3}{|c|}{{\bf +AH}}  \\ 
\hline
$d_{EL}$& {\bf PER} & {\bf FPR} & {\bf FNR} & {\bf PER}& {\bf FPR} & {\bf FNR} \\
\hline 
{\bf 0 (Original HB)} & 2.96 & 2.61 & 2.99 & 3.83 & 4.46 & 2.98\\
\hline 
{\bf 1} & 2.71 & 2.13 & 2.99 & 4.16& 5.07 & 2.97 \\
\hline 
{\bf 2} & 2.33 & 1.47 & 3.00 & 3.74 & 4.31 & 2.98 \\
\hline 
{\bf 3} & \textbf{2.31} & \textbf{1.48} & \textbf{2.98} & 3.64 & 4.16 & 2.96 \\
\hline 
{\bf 4} & 2.42 & 1.61& 2.92 & \textbf{3.49} & \textbf{4.01} & \textbf{2.95} \\
\hline 
{\bf 5} & 2.46 & 1.54 & 2.97 & 4.03 & 4.95 & 2.98\\
\hline 
{\bf 6} & 2.58&  1.71 & 2.93 & 4.35 & 5.34 & 2.99 \\
\hline
\end{tabular}
\end{center}
\end{table}

\subsubsection{Effect of $\alpha$ on performance (Table \ref{tab:Alpha})}
As $\alpha$ in Assumption \ref{ass:1} changes, optimized partitioning parameter $m_i = \lceil{T^{1/(2\alpha + d_i)}} \rceil$ changes. In illustrative results, we set $T=10000$, $M=3$ and  $d_i=10$ for all LLs. Thus, if $\alpha \geq 1.65$, $m_i = 2$. Otherwise, $m_i = 3$. 
Table \ref{tab:Alpha} shows that the optimal performance is achieved when $m_i = 2$.

\begin{table}[h!]
	\caption{Performance of HB for different $\alpha$ values.}
	\label{tab:Alpha}
	\begin{center}
		\begin{tabular}{| c | c | c | c | c | c | c | c |}
			\hline
			\multicolumn{2}{|c|}{\bf Units(\%)} & \multicolumn{3}{|c|}{{\bf HB (WM)}} &  \multicolumn{3}{|c|}{{\bf HB (AH)}} \\ 
			\hline
			$\alpha$ & $m_i$ & PER & FPR & FNR & PER & FPR & FNR \\
			\hline 
			$\geq$ 1.65 & 2 & 2.96 & 2.61 & 2.99 & 3.83 & 4.46 & 2.98 \\
			\hline 
			$<$ 1.65 & 3 & 5.56 & 6.81 & 2.98 & 6.46 & 9.02 & 2.97 \\
			\hline
		\end{tabular}
	\end{center}
\end{table}

\section{Related Works} \label{sec:related}

In this section, we compare our proposed method with other online
learning and ensemble learning methods in terms of the underlying
assumptions and performance bounds.

\textbf{Heterogeneous data observations}: Most of the existing ensemble
learning methods assume that the LLs make predictions by observing
the same set of instances \cite{littlestone1989weighted}, \cite{fan1999application},
\cite{masud2009integrating}, \cite{street2001streaming}, \cite{minku2012ddd}.
Our methods allow the LLs to act based on heterogeneous data streams
that are related to the same event. Moreover, we impose no statistical
assumptions on the correlation between these data streams. This is
achieved by isolating the decision making process of the EL from the
data. Essentially, the EL acts solely based on the predictions it
receives from the LLs.

Our proposed method can be viewed as attribute-distributed learning
\cite{zheng2011attribute, yubig2013}. In attribute-distributed learning, learners
observe different features of the same instance and make local predictions.
These local predictions are merged into a global prediction by a fusion
center (EL). Numerous papers have considered the attribute-distributed
learning model and proposed collaborative training algorithms to train
the LLs \cite{shutin2010space,hershberger2001distributed}. However,
these algorithms require information exchange between the LLs. In
contrast to these works, in our proposed work, information exchange
is only possible between an LL and the EL. \rev{Hence concerns about data
security and privacy are ruled out in our work.}

\rev{
There is a wide range of literature that develops distributed estimation techniques in which distributed LLs come up with consensus-based \cite{mateos2010distributed} or diffusion-based \cite{chen2012diffusion} parameter estimates by iteratively exchanging their local parameters computed based on the local observations. Unlike these works, in which the optimal parameter estimation problem is formulated as a distributed optimization problem, in our work the optimal prediction rule selection problem is formulated as a learning problem, and we explicitly focus on balancing the tradeoff between exploration and exploitation. Moreover, we do not make any restriction on the type of classifiers (prediction rules) used by LLs (except the similarity assumption), and do not require any message exchange between LLs. 
}


\textbf{Data-dependent oracle benchmark vs. empirical risk minimization}:
Our method can be viewed as online supervised learning with bandit
feedback, because only the estimated accuracies of the prediction
rules chosen by the LLs can be updated after the label is observed.
Most of the prior works in this field use empirical risk minimization
techniques \cite{vapnik1992principles, kivinen2004online} to learn the optimal hypothesis. 
Let $H_i := {\cal X}_i \rightarrow {\cal F}_i$ denote a hypothesis for LL $i$, which is simply a mapping from the data instance that LL $i$ observes to the set of prediction rules of LL $i$. Since the data instance space is taken as $[0,1]^{d_{i}}$, there are infinitely many hypothesis. The optimal hypothesis for LL $i$ is $H^*_i(x_i) = f^*_i(x_i)$. 
%
%
 As opposed to our work, 
 ERM assumes access to $N$ i.i.d. samples of the data instances, the label and the predictions (given as $\{ (\bs{x}(t), y(t), \{ \hat{y}_f(t) \}_{ f\in{\cal F} } ) \}_{t=1}^{N}$) by which the loss of any hypothesis $H_{i}$ can be evaluated. Using these i.i.d. samples, the empirical risk of $H_i$ is calculated as 
 $\overline{Risk}(H_i) 
 = \frac{1}{N} \sum_{t=1}^{N} [R_{f^*_i (x_i(t) ) } (t) - R_{ H_i ( x_i (t) ) }(t) ]$. 
 For LL i, ERM seeks out to find a hypothesis $\hat{H}_i$ such that 
 $\hat{H}_i \in \argmin_{h \in {\cal H}_i } \overline{Risk}(h)$.
  
 There are several important differences between ERM and our approach: In our approach the LLs and the EL update their hypothesis on-the-fly as more data and observations are gathered. IUP is an alternative to solving for the hypothesis that minimizes the empirical risk at each time step. Moreover, our algorithms are:
 (i) guaranteed to converge to the optimal hypothesis, and the convergence rate is explicitly characterized in terms of the regret bounds; 
 (ii) work efficiently even when the hypothesis space ${\cal H}_i$ is infinite or very large by partitioning ${\cal X}_i$; 
 (iii) work under partial feedback, i.e., only the prediction of the selected prediction rule is observed, hence the samples available at time step $t$ are $( \bs{x}(t), y(t), \{ \hat{y}_{ a_i(t) }(t) \}_{ i \in {\cal M} } \} )$. Moreover, not all of these are observed by the same learner.

\textbf{Reduced computational and memory complexity}: Most ensemble learning
methods require access to the entire dataset \cite{freund1995desicion} or process
the data in chunks \cite{fan1999application}. For instance, \cite{fan1999application}
considers an online version of AdaBoost, in which weights are updated
in a batch fashion after the processing of each data chunk is completed.
Unlike this work, our method processes each instance only once upon
its arrival, and do not need to store any past instances. Moreover,
the LLs only learn from their own instances, and no data exchange
between LLs are necessary. The above properties make our method efficient in terms of memory and computation, and suitable for distributed implementation.

\textbf{Decentralized consensus optimization (DCO):}
The goal in DCO is to maximize a global objective function subject to numerous local constraints \cite{nedic2009distributed,nedic2010constrained,yuan2011distributed,chang2014distributed}. In this framework, distributed agents, which only have access to local information, exchange messages to cooperate with each other, in order to maximize the global payoff. The message exchange process continues until a predefined stopping criterion is satisfied. Unlike DCO, in our work, both local and global payoff functions are not known in advance. The LLs and the EL can only obtain noisy feedback about these payoffs, which is whether a prediction error happened or not. Moreover, the optimal actions (prediction rules) depend on the data instance (context), and hence are dynamically changing. In addition, the information only flows from the LLs to the EL, and there is only a single message exchange at each decision (time) step. Unlike maximizing the global objective function of a single-shot decision problem, our goal is to maximize the cumulative reward incurred over multiple decision steps.

\section{Conclusion}\label{sec:conc}

In this paper we proposed a new online learning method that jointly considers the learning problem of the LLs and the EL. The proposed method comes with confidence and regret guarantees, which is very important in practice for many applications. Our theoretical results show that the time order of the regret for the EL is not affected by the number of LLs, which implies that the convergence speed of the EL to the optimal remains almost unchanged when the number of LLs in the system is increased. 
Our extensive numerical results show the superiority of the proposed approach in terms of its predictive accuracy. Specifically, Contextual EL performs significantly better than other ensemble learning methods, since it can utilize more information about the data features. 
We also proposed various other extensions to our proposed methods to deal with low confidence predictions during explorations and adaptation to missing labels. 

\appendices


\section{Proof of Theorem \ref{thm:Ahedgeregret}}\label{app:hedgeproof}

\rev{The proof is similar, in spirit, to that of algorithm Exp3 without mixing in \cite{bubeck2010jeux}.} Since AH observes the rewards of all prediction rules, it uses the actual cumulative losses instead of unbiased estimates of losses that is used by Exp3. 
Recall that the loss of LL $i$ at time $t$ is defined as $l_{i}(t) = 1 - v_{i}(t)$. Both $v_{i}(t)$ and $l_{i}(t)$ are in $\{0,1\}$. However, our analysis below holds for any $v_{i}(t) \in [0,1]$, hence the regret bound we derive for AH will still hold when the rewards are bounded to be in $[0,1]$.

For any $i \in {\cal M}$ we have
\jinc{
\begin{align}
& \sum_{t=1}^T v_{i}(t) - \sum_{t=1}^T \mr{E}_{I(t) \sim \bs{q}(t) } [ v_{I(t)}(t) ]  \notag \\
& = \sum_{t=1}^T   \mr{E}_{I(t) \sim \bs{q}(t) } [ l_{I(t)}(t) ] - \sum_{t=1}^T l_{i}(t) . \label{eqn:1}
\end{align}
}
\sinc{
\begin{align}
& \sum_{t=1}^T v_{i}(t) - \sum_{t=1}^T \mr{E}_{I(t) \sim \bs{q}(t) } [ v_{I(t)}(t) ]  \notag \\
&= \sum_{t=1}^T (1 - l_{i}(t) ) - \sum_{t=1}^T  (1 - \mr{E}_{I(t) \sim \bs{q}(t) } [ l_{I(t)}(t) ] )  \notag \\
& = \sum_{t=1}^T   \mr{E}_{I(t) \sim \bs{q}(t) } [ l_{I(t)}(t) ] - \sum_{t=1}^T l_{i}(t) . \label{eqn:1}
\end{align}
}
$\mr{E}_{I(t) \sim \bs{q}(t) } [ l_{I(t)}(t) ] $ can be rewritten in the following way:
\jinc{
\begin{align}
& \mr{E}_{I(t) \sim \bs{q}(t) } [ l_{I(t)}(t) ]  \notag \\
&= \frac{1}{\eta(t)} 
\log \left(   \mr{E}_{J \sim \bs{q}(t) } [ \exp( -\eta(t) ( l_{J}(t) - \mr{E}_{I(t) \sim \bs{q}(t) } [ l_{I(t)}(t) ]  ) ) ]  \right) \notag \\
&- \frac{1}{\eta(t)} \log \left( \mr{E}_{J \sim \bs{q}(t) } [ \exp( -\eta(t) l_{J}(t) ) ] \right) \label{eqn:2}
\end{align}
}
\sinc{
\begin{align}
& \mr{E}_{I(t) \sim \bs{q}(t) } [ l_{I(t)}(t) ]  
= 
\frac{1}{\eta(t)} \log \left(  \exp( \eta(t) \mr{E}_{I(t) \sim \bs{q}(t) } [ l_{I(t)}(t) ]  )  \right. \notag \\ 
& \left. \frac{ \mr{E}_{J \sim \bs{q}(t) } [ \exp( -\eta(t) l_{J}(t) ) ] }
{\mr{E}_{J \sim \bs{q}(t) } [ \exp( -\eta(t) l_{J}(t) ) ] }  \right) \notag \\
&= \frac{1}{\eta(t)} 
\log \left(   \mr{E}_{J \sim \bs{q}(t) } [ \exp( -\eta(t) ( l_{J}(t) - \mr{E}_{I(t) \sim \bs{q}(t) } [ l_{I(t)}(t) ]  ) ) ]  \right) \notag \\
&- \frac{1}{\eta(t)} \log \left( \mr{E}_{J \sim \bs{q}(t) } [ \exp( -\eta(t) l_{J}(t) ) ] \right) \label{eqn:2}
\end{align}
}
Below, we will bound the first and second term in \eqref{eqn:2} separately. For the first term we have
\jinc{
\begin{align}
& \log \left(   \mr{E}_{J \sim \bs{q}(t) } [ \exp( -\eta(t) ( l_{J}(t) - \mr{E}_{I(t) \sim \bs{q}(t) } [ l_{I(t)}(t) ]  ) ) ]  \right)    \notag \\ 
&= \log \left( \mr{E}_{J \sim \bs{q}(t) } [ \exp( -\eta(t)  l_{J}(t) ) ]  \right) + \eta(t) \mr{E}_{I(t) \sim \bs{q}(t) } [ l_{I(t)}(t) ] \notag \\
& \leq \mr{E}_{J \sim \bs{q}(t) } [ \exp( -\eta(t)  l_{J}(t) ) ] - 1
 + \eta(t) \mr{E}_{I(t) \sim \bs{q}(t) } [ l_{I(t)}(t) ] \notag \\
& = \mr{E}_{J \sim \bs{q}(t) } [ \exp( -\eta(t)  l_{J}(t) )  - 1 + \eta(t) l_{J}(t) ]  \notag \\
&\leq \mr{E}_{J \sim \bs{q}(t) } \left[ \eta^2(t) l^2_{J}(t)/  2 \right]  
 \leq  \eta^2(t) / 2 . \label{eqn:3}
\end{align}
where the first inequality follows from $\log x \leq x-1$ for $x \geq 0$, the second inequality follows from $\exp(-x) -1 + x \leq x^2 / 2$ for $x \geq 0$ and the third inequality follows from $l_{J}(t) \in [0,1]$ a.s.
}
\sinc{
\begin{align}
& \log \left(   \mr{E}_{J \sim \bs{q}(t) } [ \exp( -\eta(t) ( l_{J}(t) - \mr{E}_{I(t) \sim \bs{q}(t) } [ l_{I(t)}(t) ]  ) ) ]  \right)    \notag \\ 
&= \log \left( \mr{E}_{J \sim \bs{q}(t) } [ \exp( -\eta(t)  l_{J}(t) ) ]  \right) + \eta(t) \mr{E}_{I(t) \sim \bs{q}(t) } [ l_{I(t)}(t) ] \notag \\
& \leq \underbrace{\mr{E}_{J \sim \bs{q}(t) } [ \exp( -\eta(t)  l_{J}(t) ) ] - 1}_{ \text{since } \log x \leq x-1 \text{ for } x \geq 0   }
 + \eta(t) \mr{E}_{I(t) \sim \bs{q}(t) } [ l_{I(t)}(t) ] \notag \\
& = \underbrace{\mr{E}_{J \sim \bs{q}(t) } [ \exp( -\eta(t)  l_{J}(t) )  - 1 + \eta(t) l_{J}(t) ] }_{ \text{by linearity of expectation} } \notag \\
&\leq \underbrace{\mr{E}_{J \sim \bs{q}(t) } \left[ \eta^2(t) l^2_{J}(t)/  2 \right] }_{ \text{since }
\exp(-x) -1 + x \leq x^2 / 2 \text{ for } x \geq 0} \notag \\
& \leq \underbrace{ \eta^2(t) / 2 }_{\text{since } l_{J}(t) \in [0,1] \text{ almost surely }}. \label{eqn:3}
\end{align}
}
Let
\jinc{
$\Phi_t(\eta) := \frac{1}{\eta} \log \frac{1}{M} \sum_{i=1}^M \exp( -\eta L_{i}(t) )$.
}
\sinc{
\begin{align}
\Phi_t(\eta) := \frac{1}{\eta} \log \frac{1}{M} \sum_{i=1}^M \exp( -\eta L_{i}(t) ) .     \notag
\end{align}
}
For the second term of \eqref{eqn:2} we have
\begin{align}
& - \frac{1}{\eta(t)} \log \left( \mr{E}_{J \sim \bs{q}(t) } [ \exp( -\eta(t) l_{J}(t) ) ] \right)  \notag \\
& = - \frac{1}{\eta(t)} \log \frac{ \sum_{i=1}^M \exp( -\eta(t) l_{i}(t) ) \exp( -\eta(t) L_{i}(t-1) ) }
{  \sum_{j=1}^M \exp( -\eta(t) L_{j}(t-1) )  } \notag \\
&= - \frac{1}{\eta(t) } \log \frac{ \sum_{i=1}^M \exp( -\eta(t) L_{i}(t) ) }
{  \sum_{j=1}^M \exp( -\eta(t) L_{j}(t-1) )  } \notag \\
&= \Phi_{t-1}(\eta(t)) - \Phi_{t}(\eta(t)) \label{eqn:4}
\end{align}
Using \eqref{eqn:3} and \eqref{eqn:4} we have
\begin{align}
& \sum_{t=1}^T v_{i}(t) - \sum_{t=1}^T \mr{E}_{I(t) \sim \bs{q}(t) } [ v_{I(t)}(t) ]       \notag \\
&\leq \sum_{t=1}^T \frac{\eta(t) }{2} 
+ \sum_{t=1}^T [ \Phi_{t-1}(\eta(t)) - \Phi_{t}(\eta(t)) ]
- \sum_{t=1}^T l_{i}(t) . \label{eqn:5}
\end{align}
We will use the following properties of $\Phi_t(\eta) $ proofs of which can be found in Theorem 2.1 of \cite{bubeck2010jeux}.
\jinc{
$\textbf{P1}:  \sum_{t=1}^T  [ \Phi_{t-1}(\eta(t)) - \Phi_{t}(\eta(t))]  
=  \sum_{t=1}^{T-1}  [ \Phi_{t}(\eta(t+1)) - \Phi_{t}(\eta(t))]  - \Phi_{T}(\eta(T))$,   
$\textbf{P2}:  \frac{d}{d \eta} \Phi_t(\eta) \geq 0$.
}
\sinc{
\begin{align}
\textbf{P1}: & \sum_{t=1}^T  [ \Phi_{t-1}(\eta(t)) - \Phi_{t}(\eta(t))]  \notag \\
&=  \sum_{t=1}^{T-1}  [ \Phi_{t}(\eta(t+1)) - \Phi_{t}(\eta(t))]  - \Phi_{T}(\eta(T))     \notag \\
\textbf{P2}: & \frac{d}{d \eta} \Phi_t(\eta) \geq 0 \notag
\end{align}
}
Since $\eta(t)$ is a non-increasing sequence 
$\Phi_{t}( \eta(t+1) ) - \Phi_{t}( \eta(t) ) \leq 0$ 
by $\textbf{P2}$. Hence,
\begin{align}
& \sum_{t=1}^T  [ \Phi_{t-1}(\eta(t)) - \Phi_{t}(\eta(t)) ] 
\leq - \Phi_{T}(\eta(T))       \notag \\
& = \frac{\log M} {\eta(T)} - \frac{1} {\eta(T)} \log \left( \sum_{j=1}^M \exp(-\eta(T) L_{j}(T) ) \right) \notag \\
& \leq \frac{\log M} {\eta(T)} - \frac{1} {\eta(T)}  \log \exp ( - \eta(T) L_{i}(T) ) \notag \\
& = \frac{\log M} {\eta(T)} + \sum_{t=1}^T l_{i}(t)  . \label{eqn:6} 
\end{align}
Using \eqref{eqn:5} and \eqref{eqn:6} together with $\eta(t) = \sqrt{ \log M / t }$ and the result in Appendix \ref{app:seriesbound}, we obtain
\begin{align}
 \sum_{t=1}^T v_{i}(t) - \sum_{t=1}^T \mr{E}_{I(t) \sim \bs{q}(t) } [ v_{I(t)}(t) ]    
 & \leq \sum_{t=1}^T \frac{\eta(t) }{2}  + \frac{\log M} {\eta(T)} \notag \\
 & \leq 2 \sqrt{ T \log M } \notag .
\end{align}
Since the choice of $i$ is arbitrary, the result follows by setting $i = I_b$.

\section{Preliminaries for the proof of  Theorem \ref{thm:CUPregret}}\label{app:prelim}

All the expressions used in the proofs below are related to LL $i$. To simplify the notation, we drop subscripts/superscripts related to LL $i$ from the notation. For instance, we use $\hat{\pi}_{f,p}(t)$ instead of $\hat{\pi}^i_{f,p}(t)$, $N_{f,p}(t)$ instead of $N^i_{f,p}(t)$, $p(t)$ instead of $p_i(t)$ and \rev{$f^*$ instead of $f^*_i(x)$ when the data instance we refer to is clear from the context.}

The regret is computed by conditioning on $\bs{X}^T_i = \bs{x}^T_i$. 
Let $\tau^i_{p}(t)$ denote the time step in which the $t$th context arrives to 
$p \in {\cal P}_i$ of LL $i$. Let 
$\tilde{x}_p(t) = x_i (\tau^i_{p}(t) )$, 
$\tilde{r}_{f,p}(t) = r_f (\tau^i_{p}(t) )$,
$\tilde{v}_p(t) = v_i (\tau^i_{p}(t) )$,
$\tilde{\pi}_{f,p}(t) = \hat{\pi}_{f,p}( \tau^i_{p}(t) )$,
$\tilde{N}_{f,p}(t) = N_{f,p}( \tau^i_{p}(t) ) $ 
and
$\tilde{a}_p(t) = a_i( \tau^i_{p}(t) )$.
Let $N^i_{p}(T)$ (or simply $N_{p}(T)$) be the number of context arrivals to $p \in {\cal P}_i$ by the end of time $T$. Let
\begin{align}
C_{f,p}(t) :=  \sqrt{ \frac{2}{\tilde{N}_{f,p}(t)} (1 + 2 \log ( 2 |{\cal F}_i| (m_i)^{d_i} T^{\frac{3}{2}} ) )  }  .  \notag
\end{align}
For any $p \in {\cal P}_i$, $f \in {\cal F}_i$ and $t \in \{1, \ldots, N_p(T)\}$, we define the following lower confidence bound (LCB) and upper confidence bound (UCB):
\begin{align}
L_{f,p}(t) &:= \max\{ \tilde{\pi}_{f,p}(t) - C_{f,p}(t), 0 \}\notag \\
U_{f,p}(t) &:= \min\{ \tilde{\pi}_{f,p}(t) + C_{f,p}(t), 1 \} . \notag
\end{align}
Let
$\text{UC}(f,p,v)   
 := 
\cup_{t=1}^{ N_{p}(T) }  \{ \pi_{f}( \tilde{x}_p(t) ) \notin [L_{f,p}(t)- v, U_{f,p}(t) + v]  \}$
denote the event that LL $i$ is not confident about the accuracy of its prediction rule $f$ at least once for instances in $p$ by time $T$. Throughout our analysis we set $v = L(\sqrt{d_{i}}/m_{i})^{\alpha}$.
Let 
$\text{UC}(p,v) := \bigcup_{f \in {\cal F}_i}  \text{UC}(f,p,v)$
and
\begin{align}
\text{UC}(v) :=  \cup_{p \in {\cal P}_i} \text{UC}(p,v) .    \label{eqn:UCunionbound}
\end{align}
For each $p \in {\cal P}_i$ and $f \in {\cal F}_i$ let 
$\overline{\pi}_{f,p} := \sup_{x \in p} \pi_f(x)$ and 
$\underline{\pi}_{f,p} := \inf_{x \in p} \pi_f(x)$.

\section{Proof of Theorem \ref{thm:CUPregret}} \label{app:theorem1}

We will bound the regret in each $p \in {\cal P}_i$ separately.
Then, we will sum over all $p \in {\cal P}_i$ to bound the total regret.
Preliminaries are given in Appendix \ref{app:prelim}.
\begin{align}
 \text{Reg}_{i}(T | \bs{X}_i^T) 
&= \sum_{t=1}^T \pi_{f^*_i ( X_{i}(t) ) } ( X_{i}(t) )   \notag \\
&- \expect*{ \sum_{t=1}^{T} \pi_{ A_i(t) } (X_{i}(t)) |  \bs{X}_i^T  } . \label{eqn:transform1} 
\end{align}
The first term in \eqref{eqn:transform1} is obtained by observing that 
\begin{align}
& \expect*{ \sum_{t=1}^{T} R_{ f^*_i(X_i(t)) }(t) | \bs{X}_i^T } \notag \\
&= \sum_{t=1}^{T} \sum_{f \in {\cal F}_i } 
\expect*{  R_{ f }(t) \mr{I} ( f^*_i(X_i(t)) = f  ) | \bs{X}_i^T }   \notag \\
&= \sum_{t=1}^{T} \sum_{f \in {\cal F}_i } 
\mr{I} ( f^*_i(X_i(t)) =f  )  \expect*{ R_{ f }(t) | \bs{X}_i^T } \notag \\
&= \sum_{t=1}^{T} \pi_{ f^*_i(X_i(t)) } ( X_{i}(t) ) . \notag
\end{align}
Let ${\cal F}_{t-1}$ be the sigma field generated by $\bs{X}_i^T$, $\bs{A}_i^{t-1}$, $\bs{Y}^{t-1}$. 
The second term in \eqref{eqn:transform1} is obtained by observing that 
\begin{align}
& \expect*{ \sum_{t=1}^{T}  R_{A_i(t)}( t ) | \bs{X}_i^T }  \notag \\
& = \sum_{t=1}^{T}  \sum_{f \in {\cal F}_i} 
\expect*{ \expect*{ R_f(t) \mr{I} (A_i(t) = f ) |  {\cal F}_{t-1} } | \bs{X}_i^T } \label{eqn:conditional2} \\
& = \sum_{t=1}^{T}  \sum_{f \in {\cal F}_i} 
\expect*{ \mr{I} (A_i(t) = f ) \expect*{ R_f(t) |  {\cal F}_{t-1} } | \bs{X}_i^T } \label{eqn:conditional3} \\
&= \sum_{t=1}^{T}  \sum_{f \in {\cal F}_i} 
\expect*{ \mr{I} (A_i(t) = f )  \pi_f( X_i(t) )  | \bs{X}_i^T } \label{eqn:conditional4} \\
&= \expect*{ \sum_{t=1}^{T}   \pi_{A_i(t)}( X_i(t) )  | \bs{X}_i^T } \notag 
\end{align}   
where \eqref{eqn:conditional2} is by the law of iterated expectations, \eqref{eqn:conditional3} is by the fact that $\mr{I} (A_i(t) = f )$ is ${\cal F}_{t-1}$ measurable, \eqref{eqn:conditional4} is by definition of $\pi_f(\cdot)$ and the fact that $R_f(t)$ is independent of all random variables in $(\bs{X}_i^T, \bs{A}_i^{t-1}, \bs{Y}^{t-1})$ except $X_i(t)$. 

For $p \in {\cal P}_i$, let
\begin{align}
& \text{Reg}_{i,p}(T | \bs{X}_i^T = \bs{x}_i^T ) 
 := \sum_{t=1}^{ N_p(T) }  \pi_{f^{*}( \tilde{x}_p(t))}( \tilde{x}_p(t))  \notag \\
&- \expect*{ \sum_{t=1}^{ N_p(T) } \pi_{ \tilde{A}_p(t) } ( \tilde{x}_p (t) ) |  \bs{X}_i^T = \bs{x}_i^T } .
\label{eqn:partitionregret}
\end{align}
Using \eqref{eqn:transform1} we obtain
\begin{align}
& \text{Reg}_{i}(T | \bs{X}_i^T = \bs{x}_i^T ) 
= \sum_{p \in {\cal P}_i} 
\sum_{t=1}^{ N_p(T) }  \pi_{f^{*}( \tilde{x}_p(t))}( \tilde{x}_p(t)) \notag \\
&- \expect*{ \sum_{p \in {\cal P}_i} \sum_{t=1}^{ N_p(T) }
 \pi_{ \tilde{A}_p(t) } ( \tilde{x}_p (t) ) |  \bs{X}_i^T = \bs{x}_i^T }  \notag \\
&= \sum_{p \in {\cal P}_i}  \text{Reg}_{i,p}(T | \bs{X}_i^T = \bs{x}_i^T ) 
 \label{eqn:partitionexpregret}  .
\end{align}
The expectation in \eqref{eqn:partitionregret} is taken with respect to the randomness of $\tilde{A}_p(1), \ldots, \tilde{A}_p( N_p(T) )$ given $\bs{X}_i^T = \bs{x}_i^T$. By the definition of IUP, conditioned on $\bs{X}_i^T = \bs{x}_i^T$, $\tilde{A}_p(t)$ only depends on random variables $\tilde{A}_p(1), \tilde{V}_p(1), \ldots, \tilde{A}_p(t-1), \tilde{V}_p(t-1)$. Since, $\tilde{V}_p(t) = \tilde{R}_{\tilde{A}_p(1), p}(t)$, we conclude that $\{ \tilde{A}_p(t) \}_{t=1}^{N_p(T)}$ only depends on random variables
$\bs{R}_p := \cup_{f \in {\cal F}_i} \{ \tilde{R}_{f,p}(t) \}_{t=1}^{N_p(T)}$.
Hence, the expectation in \eqref{eqn:partitionregret} is taken with respect to the conditional distribution of 
$\bs{R}_p$ given $\bs{x}_i^T$.

Since $ \{ ( \bs{X}(t), Y(t), \{ \hat{Y}_f(t) \}_{f \in {\cal F}} ) \}_{t=1}^T$ is an i.i.d. sequence, random variables $R_f(t)$, $t=1,\ldots,T$ conditioned on $\bs{X}_i^T$ are independent. Since 
$R_f(t) \in \{0,1\}$ and $\expect{R_f(t) | \bs{X}_i^T = \bs{x}_i^T} = \pi_f(x_i(t))$, we can say that conditioned on 
$\bs{X}_i^T = \bs{x}_i^T$, $\{ R_f(t) \}_{t=1}^T$ is a sequence of independent Bernoulli random variables with parameters $ \{ \pi_f(x_i(t)) \}_{t=1}^T$ for $f \in {\cal F}_i$. With an abuse of notation, in the subsequent analysis in this section, $R_f(t)$ will done the random reward of $f$ conditioned on $X_i(t) = x_i(t)$, and all the expectations are taken with respect to the random variables defined above, unless otherwise stated. Hence, given $\bs{X}_i^T = \bs{x}_i^T$, we drop the conditioning on $\bs{X}_i^T$ from the notation and simply write
\begin{align}
\text{Reg}_{i,p}(T) &= \hspace{-0.1in}  \sum_{t=1}^{ N_p(T) } \hspace{-0.1in}  \pi_{f^{*}( \tilde{x}_p(t))}( \tilde{x}_p(t)) 
- \mr{E}_{ \bs{R}_p  } 
\left [ \sum_{t=1}^{ N_p(T) } \hspace{-0.1in} \pi_{ \tilde{A}_p(t) } ( \tilde{x}_p (t) )  \right] . \notag
\end{align}

By the law of total expectation we have
\begin{align}
& \expect{ \text{Reg}_{i,p}(T) }
= \expect{ \text{Reg}_{i,p}(T) | \text{UC}(p,v) } \Pr ( \text{UC}(p,v)   )  \notag \\
&+ \expect{ \text{Reg}_{i,p}(T) | \text{UC}^C(p,v) } \Pr ( \text{UC}^C(p,v)   )    \notag \\
&\leq T \Pr ( \text{UC}(p,v)   ) + \expect{ \text{Reg}_{i,p}(T) | \text{UC}^C(p,v) } \label{eqn:regdecompose}.
\end{align}
We will use the results of following lemmas to upper bound \eqref{eqn:regdecompose}.

\begin{lemma} \label{lemma:firstterm}
\textbf{(Bound on $\Pr ( \text{UC}(f,p,v)   )$)} 
$\Pr ( \text{UC}(f,p,v)   ) \leq  1 / ( |{\cal F}_i| m_i^{d_i} T ) $.
\end{lemma} 
\begin{proof} 
Equivalently, we can define $\{  \tilde{R}_{f,p}(t) \}_{t=1}^{ N_{\rho}(T) }$ in the following way: 
Let $\{ \eta_t \}_{t=1}^{ N_{\rho}(T) }$ be a sequence of i.i.d. random variables uniformly distributed in $[0,1]$. Then, 
$\tilde{R}_{f,p}(t) = \mr{I} ( \eta_t \leq  \pi_{f} ( \tilde{x}_p(t))  )$.
We can express the sample mean reward (accuracy) of $f$ as 
$\tilde{\pi}_{f,p}(t) =  \sum_{l=1}^{\rev{t-1}} \tilde{R}_{f,p}(l) 
 \text{I} ( \tilde{a}_p(l) = f )   /  \tilde{N}_{f,p}(t)$.
From the definitions of $L_{f,p}(t)$,$U_{f,p}(t)$ and $\text{UC}(f,p,v)$, it can be observed that the event $\text{UC}(f,p,v)$ happens when $\tilde{\pi}_{f,p}(t)$ remains close to (or concentrates around) $\pi_f(\tilde{x}_p(t))$ for all $t \in \{1,\ldots,N_p(T) \}$. 

This motivates us to use the concentration inequality given in Appendix \ref{app:concentration}, which is derived in \cite{russo2014learning} from a similar concentration inequality in \cite{abbasi2011improved}.
This inequality requires the expected reward from an action to be equal to the same constant at all time steps. This is clearly not the case for $\pi_f(\tilde{x}_p(t))$ since elements of \orev{$\{ \tilde{x}_p(t) \}_{t=1}^{N_p(T)}$} are not identical which makes distributions of $\tilde{R}_{f,p}(t) $, $t \in \{1, \ldots, \orev{N_p(T)} \}$ different.

In order to overcome this issue, we propose a novel {\em sandwich technique}. 
Based on $\eta_t$, we define two new sequences of random variables, whose sample mean values will lower and upper bound $\tilde{\pi}_{f,p}(t)$. The {\em best sequence} is defined as $\{  \bar{R}_{f,p}(t) \}_{t=1}^{ N_{\rho}(T) }$ where
$ \overline{R}_{f,p}(t) = \mr{I} ( \eta_t \leq  \overline{\pi}_{f,p} )$,
and the {\em worst sequence} is defined as 
$\{  \underline{R}_{f,p}(t) \}_{t=1}^{ N_{\rho}(T) }$ where
$\underline{R}_{f,p}(t) = \mr{I} ( \eta_t \leq  \underline{\pi}_{f,p} ) $.
Let 
$\overline{\pi}_{f,p}(t) := \sum_{l=1}^{t-1} \overline{R}_{f,p}(l)  \text{I} ( \tilde{a}_p(l) = f ) /  \tilde{N}_{f,p}(t)$ and
$\underline{\pi}_{f,p}(t) := \sum_{l=1}^{t-1} \underline{R}_{f,p}(l)  \text{I} ( \tilde{a}_p(l) = f )   / \tilde{N}_{f,p}(t)$.
We have
$\underline{\pi}_{f,p}(t) \leq \tilde{\pi}_{f,p}(t) \leq  \overline{\pi}_{f,p}(t)  
~~\forall t \in \{1,\ldots,N_p(T)  \}$
almost surely. Let
$\overline{L}_{f,p}(t) := \max \{ \overline{\pi}_{f,p}(t) - C_{f,p}(t), 0 \}$,     
$\overline{U}_{f,p}(t) := \min \{ \overline{\pi}_{f,p}(t) + C_{f,p}(t) , 1 \}$,  
$\underline{L}_{f,p}(t) := \max \{ \underline{\pi}_{f,p}(t) - C_{f,p}(t), 0 \}$ and 
$\underline{U}_{f,p}(t) := \min \{ \underline{\pi}_{f,p}(t) + C_{f,p}(t), 1 \}$.
It can be shown that 
\begin{align}
& \{ \pi_{f}( \tilde{x}_p(t) ) \notin [L_{f,p}(t)- v, U_{f,p}(t) + v]  \}  \notag \\
& \subset \{ \pi_{f}( \tilde{x}_p(t) ) 
\notin [\overline{L}_{f,p}(t)  - v, 
            \overline{U}_{f,p}(t)  + v]  \}  \notag \\
&\cup  \{ \pi_{f}( \tilde{x}_p(t) ) 
\notin [\underline{L}_{f,p}(t)   - v, 
            \underline{U}_{f,p}(t)  + v]  \} . \notag
\end{align}
The following inequalities can be obtained from Assumption \ref{ass:1}.
\begin{align}
& \pi_f( \tilde{x}_p(t) )   \leq \overline{\pi}_{f,p} \leq \pi_f( \tilde{x}_p(t) ) + L \left( \frac{\sqrt{d_i}}{m_i} \right)^{\alpha}  \label{eqn:bestbound} \\
& \pi_f( \tilde{x}_p(t) ) - L \left( \frac{\sqrt{d_i}}{m_i} \right)^{\alpha}   \leq \underline{\pi}_{f,p} \leq \pi_f( \tilde{x}_p(t) ) . \label{eqn:worstbound}  
\end{align}
Since $v = L \left( \sqrt{d_i}/m_i \right)^{\alpha}$, using \eqref{eqn:bestbound} and \eqref{eqn:worstbound} it can be shown that
\begin{align}
& \{ \pi_{f}( \tilde{x}_p(t) ) \notin [\overline{L}_{f,p}(t)  - v, \overline{U}_{f,p}(t)  + v]  \}       \notag \\
\subset & \{ \overline{\pi}_{f,p} \notin [\overline{L}_{f,p}(t)  , \overline{U}_{f,p}(t) ]  \},
\text{ and }    \notag \\
& \{ \pi_{f}( \tilde{x}_p(t) ) \notin [\underline{L}_{f,p}(t)  - v, \underline{U}_{f,p}(t)  + v]  \}       \notag \\
\subset & \{ \underline{\pi}_{f,p} \notin [\underline{L}_{f,p}(t)  , \underline{U}_{f,p}(t) ]  \}  .  \notag
\end{align}
Using the equation above and the union bound we obtain
\begin{align}
\Pr( \text{UC}(f,p,v) ) 
&\leq \Pr\left( \cup_{t=1}^{ N_{p}(T) } \{ \overline{\pi}_{f,p}  \notin [\overline{L}_{f,p}(t) , \overline{U}_{f,p}(t) ]  \}  \right)   \notag \\
&+ \Pr \left( \cup_{t=1}^{ N_{p}(T) } \{ \underline{\pi}_{f,p} \notin [\underline{L}_{f,p}(t)  , \underline{U}_{f,p}(t) ]  \} \right) . \notag
\end{align}
Both terms on the right-hand side of the inequality above can be bounded using the concentration inequality in Appendix \ref{app:concentration}. Using 
$\delta = 1/ (2 |{\cal F}_i| (m_i)^{d_i} T)$ in Appendix \ref{app:concentration} gives
$\Pr( \text{UC}(f,p,v) ) \leq 1 / (  |{\cal F}_i| (m_i)^{d_i} T  )$
since \rev{$1 + N_{f,p}(T) \leq T$}. 
\end{proof}

\begin{lemma} \label{lemma:indexdiff}
On event $\text{UC}^C(p,v)$ we have
$\pi_{f^*}( \tilde{x}_p(t) ) - \pi_{\tilde{a}_p(t)} (\tilde{x}_p(t) ) \leq U_{\tilde{a}_p(t),p}(t) - L_{\tilde{a}_p(t),p}(t) + \rev{ 2 v }$
for all  $t \in \{1,\ldots, \orev{ N_p(T) } \}$.
\end{lemma} 
\begin{proof}
$U_{\tilde{a}_p(t),p}(t) + v \geq U_{f^*,p}(t) + v$
since IUP selects the decision rule with the highest index at each time step. On event $\text{UC}^C(p,v)$ this implies
$\pi_{f^*}(\tilde{x}_p(t)) \leq U_{\tilde{a}_p(t),p}(t) + v$. The proof concludes by observing that 
$\pi_{\tilde{a}_p(t)}(\tilde{x}_p(t)) \geq L_{\tilde{a}_p(t),p}(t) - v$  
on event $\text{UC}^C(p,v)$.
\end{proof}

\begin{lemma} \label{lemma:secondterm}
\textbf{(Bound on $\expect{ \text{Reg}_{i,p}(T) | \text{UC}^C(p,v) }$)}
\begin{align}
& \expect{ \text{Reg}_{i,p}(T) | \text{UC}^C(p,v) } \notag \\
& \leq  2 v N_{p}(T) + 2 A_{m_i} \sqrt{ |{\cal F}_i| N_p(T)  } +  |{\cal F}_i|   \notag
\end{align}
where
$A_{m_i} := 2 \sqrt{ 2( 1 + 2 \log (2 |{\cal F}_i| (m_i)^{d_i} T^{\frac{3}{2}}   ) }$.
\end{lemma} 
\begin{proof}
Let ${\cal T}_{f,p} := \{ t \leq N_p(T) : \tilde{a}_p(t) = f  \}$.
By Lemma \ref{lemma:indexdiff},
\begin{align}
& \mr{E} [ \text{Reg}_{i,p}(T) | \text{UC}^C(p,v) ] \leq  2 v N_{p}(T)  \notag \\
&\hspace{-0.1in} + \expect*{ \sum_{ f \in {\cal F}_i }  \sum_{ t \in {\cal T}_{f,p} } 
\hspace{-0.05in} ( U_{\tilde{A}_p(t),p}(t) - L_{\tilde{A}_p(t),p}(t)  ) \hspace{-0.05in} | \hspace{-0.05in} \text{UC}^C(p,v)   } . \label{eqn:dec1}
\end{align}
Next, we show
\begin{align}
& \sum_{f \in {\cal F}_i} \sum_{ t \in {\cal T}_{f,p} } ( U_{\tilde{a}_p(t),p}(t) - L_{\tilde{a}_p(t),p}(t)  )      \notag \\
& \leq   \sum_{f \in {\cal F}_i} \left( 1 + A_{m_i} 
\sum_{ \{ t \in {\cal T}_{f,p}: \tilde{N}_{f,p}(t) \geq 1 \}} 
\sqrt{ \frac{1}{\tilde{N}_{f,p}(t) } } \right)
 \label{eqn:cauchyineq1} \\
& =  |{\cal F}_i| + A_{m_i} \sum_{f \in {\cal F}_i} \sum_{ l=0}^{ \rev{N_{f,p}(T) - 1}  } \sqrt{ \frac{1}{1 +l} }   \notag \\
& \leq |{\cal F}_i| + 2 A_{m_i} \sum_{f \in {\cal F}_i} \sqrt{ N_{f,p}(T) }  \label{eqn:dec2}  \\
& \leq |{\cal F}_i| + 2 A_{m_i} \sqrt{ |{\cal F}_i| N_p(T)  } \label{eqn:cauchyineq2}
\end{align}
where \eqref{eqn:cauchyineq1} follows from the definition of $L_{f,p}(t)$ and $U_{f,p}(t)$, \eqref{eqn:dec2} follows from the fact that 
$\sum_{ l=0}^{ \rev{N_{f,p}(T) - 1}  } \sqrt{ \frac{1}{1 +l} } 
\leq \int_{x=0}^{N_{f,p}(T)} (1 / \sqrt{x}) dx = 2 \sqrt{ N_{f,p}(T)  }$
and \eqref{eqn:cauchyineq2}
is obtained by applying the Cauchy-Schwarz inequality given in Appendix \ref{app:cauchy} and observing that $N_p(T) \rev{\geq} \sum_{f \in {\cal F}_i} N_{f,p}(T)$.
\end{proof}

Lemma \ref{lemma:firstterm} and the union bound yields
\begin{align}
\Pr ( \text{UC}(p,v)  ) \leq 1 / ( (m_i)^{d_i} T ) .   \label{eqn:dec3}
\end{align}
Upper bounding \eqref{eqn:regdecompose} by Lemma \ref{lemma:secondterm} and \eqref{eqn:dec3} gives
\begin{align}
\mr{E} [ \text{Reg}_{i,p}(T) ] 
& \leq \frac{1}{  (m_i)^{d_i} }  + 2 v N_{p}(T)  + 2 A_{m_i}   \sqrt{ |{\cal F}_i|  N_p(T) } \notag \\
&+ |{\cal F}_i| . \label{eqn:dec4}
\end{align}
Using \eqref{eqn:partitionexpregret} together with \eqref{eqn:dec4} results in 
\begin{align}
& \text{Reg}_i(T | \bs{X}^T_i = \bs{x}^T_i ) \notag \\
& \leq \sum_{p \in {\cal P}_i} 
 \left( \frac{1}{  (m_i)^{d_i} }  + 2 v N_{p}(T)  + |{\cal F}_i| + 2 A_{m_i}   \sqrt{ |{\cal F}_i| N_p(T) } \right)     \notag \\
 &\leq 1+ 2v  T +  |{\cal F}_i| (m_i)^{d_i} + 2 A_{m_i}  \sqrt{ |{\cal F}_i| (m_i)^{d_i} T } \label{eqn:thm1last}
\end{align}
where the last inequality follows from the Cauchy-Schwarz inequality and $\sum_{p \in {\cal P}_i} N_p(T) = T$. 
The result of the theorem is obtained from \eqref{eqn:thm1last} by setting $m_i = \lceil T^{1/(2\alpha + d_i)}  \rceil$.

\section{Proof of Corollary \ref{cor:confidencebound}}\label{app:corollary1}

Using \eqref{eqn:dec3}, \eqref{eqn:UCunionbound} and the union bound, we obtain (for any LL $i$)
$\Pr \left( \text{UC}^C \left( L (  \frac{\sqrt{d_i}}{m_i}  )^{\alpha}   \right) \right)
\geq 1 - \frac{1}{T}$.
Lemma \ref{lemma:indexdiff} states that on event $\text{UC}^C \left( L (  \frac{\sqrt{d_i}}{m_i}  )^{\alpha}   \right)$ we have
\begin{align}
&\pi_{f^*_i(x_i(t))}( x_i(t)) - \pi_{a_i(t)} ( x_i(t)) \notag \\
&\leq U_{a_i(t),p_i(t)}(  N_{p_i(t)}(t)    ) - L_{a_i(t),p_i(t)}( N_{p_i(t)}(t)   ) \notag \\
&+ 2 L (  \frac{\sqrt{d_i}}{m_i}  )^{\alpha} .  \notag      
\end{align}
The result follows from the definition of $U_{f,p}(t)$, $L_{f,p}(t)$ and the fact that $m_i = \lceil T^{1/(2\alpha+d_i)} \rceil$.

\section{Proof of Theorem \ref{thm:regretensemble}}\label{app:theorem2}

Since the result of Theorem \ref{thm:Ahedgeregret} holds for any realization $\{ \bs{v}^T_i \}_{i \in {\cal M}}$ of the reward sequence, for any distribution over the reward sequence, we have
\begin{align}
\expect*{ \sum_{t=1}^T V_{i^*}(t) }
- \expect*{ \sum_{t=1}^T R_{\text{EL}}(t) }
\leq 2 \sqrt{T \log M} . \label{eqn:step1}
\end{align}
The equation above holds since
$\expect{ \max_{i \in {\cal M}} \sum_{t=1}^T V_i(t) }
 \geq \expect{ \sum_{t=1}^T V_{i^*}(t) }$
for any distribution over the reward sequence. $\overline{\text{Reg}}_{\text{EL}}(T)$ can be re-written as 
\begin{align}
\hspace{-0.05in} \overline{\text{Reg}}_{\text{EL}}(T) &= 
\expect*{  \sum_{t=1}^T  \hspace{-0.05in}    R_{ f^*_{i^*} ( X_{i^*}(t) ) }(t) }
\hspace{-0.05in} - \expect*{  \sum_{t=1}^T  \hspace{-0.05in} R_{ A_{i^*}(t) }(t) } \label{eqn:regEL1}  \\
& +  \expect*{ \sum_{t=1}^T  V_{ i^* }(t) }
- \expect*{ \sum_{t=1}^T R_{\text{EL}}(t) } \label{eqn:regEL2}
\end{align}
since 
$ \expect{ \sum_{t=1}^T  R_{ A_{i^*}(t) }(t) }
 = \expect{ \sum_{t=1}^T  V_{ i^* }(t) }$.
The result follows from bounding \eqref{eqn:regEL1} by using Theorem \ref{thm:CUPregret}, and \eqref{eqn:regEL2} by \eqref{eqn:step1}.

\section{Proof of Theorem \ref{thm:CHregret} }\label{app:CHregret}

Since CH keeps and updates a separate probability distribution over the LLs for each $p \in {\cal P}_{\text{EL}}$, regret given in \eqref{eqn:newregret} can be re-written as
\begin{align}
\text{Reg}_{\text{CEL}}(T) = \sum_{ p \in {\cal P}_{\text{EL}} } \left[
\sum_{l \in {\cal Z}_p(T) }  v_{ i^*_p }(l)
-  \mr{E} \left[ \sum_{l \in {\cal Z}_p(t) } R_{\text{EL}}(l)  \right]      \right]   .     \notag
\end{align}
By Theorem \ref{thm:Ahedgeregret} we obtain
\begin{align}
& \sum_{ l \in {\cal Z}_p(T) } \hspace{-0.15in}   v_{ \orev{i^*_p} }(l)  
- \mr{E} \left[ \sum_{ l \in {\cal Z}_p(T) } \hspace{-0.15in}  R_{\text{EL}}(l)  \right] 
 \leq 2 \sqrt{N_{\text{EL},p}(T) \log M}   . \label{eqn:CEL1}
\end{align}
Using \eqref{eqn:CEL1}, the Cauchy-Schwarz inequality given in Appendix \ref{app:cauchy} and the fact that $\sum_{p \in {\cal P}_{\text{EL}}} N_{\text{EL},p}(T)   = T$ we get
\begin{align}
& \sum_{ p \in {\cal P}_{\text{EL}} } \left[
\sum_{l \in {\cal Z}_p(T) }  v_{ \orev{ i^*_p } }(l)
-  \mr{E} \left[ \sum_{l \in {\cal Z}_p(T) } R_{\text{EL}}(l)  \right]      \right] \notag \\
& \leq  2 \sqrt{\log M} \sum_{p \in {\cal P}_{\text{EL}}} \sqrt{N_{\text{EL},p}(T) } 
\leq 2 \sqrt{ T ( m_{\text{EL}} )^{d_{\text{EL}}}  \log M } . \notag
\end{align}

\vspace{-0.2in}
\section{Concentration Inequality (Appendix A in \cite{russo2014learning})}\label{app:concentration} 

Consider a prediction rule $f$ of LL $i$ for which the rewards are generated by an i.i.d. process $\{ R(t) \}_{t=1}^T$ with $\pi_f = \mr{E} [R(t)]$, where the noise $R(t) - \pi_f$ is bounded in $[-1,1]$. Let $N_f(T) \geq 1$ denote the number of times $f$ is selected by LL $i$ by the end of time $T$. 
Let $\hat{\pi}_{f}(T) = \sum_{t=1}^T \mr{I} (a_i(t) =f ) R(t) / N_f(T)$
For any $\delta > 0$ with probability at least $1-\delta$ we have
\begin{align}
&\left| \hat{\pi}_{f}(T)  - \pi_f \right| \notag \\
& \leq \sqrt{  \frac{2}{N_f(T)} 
\left(       
1 + 2 \log \left(  \frac{ (1 + N_f(T) )^{1/2} } {\delta}    \right)  
 \right)  }  ~~ \forall T \in \mathbb{N}.   \notag
\end{align}

\vspace{-0.2in}
\section{Cauchy-Schwarz Inequality}\label{app:cauchy}

$|<\bs{x}, \bs{y}>| \leq \sqrt{ <\bs{x}, \bs{x}> <\bs{y}, \bs{y} >}$, where $\bs{x}$ and $\bs{y}$ are $D$-dimensional real-valued vectors and $<\cdot,\cdot>$ denotes the standard inner product.

\section{A bound on divergent series} \label{app:seriesbound}
For $\rho>0$, $\rho \neq 1$,
%
$\sum_{t=1}^{T} 1/(t^\rho) \leq 1 + (T^{1-\rho} -1)/(1-\rho)$.
%
\begin{proof}
See \cite{chlebus2009approximate}.
\end{proof}

\bibliographystyle{IEEE}
\bibliography{OSA}

\begin{thebibliography}{10}
\providecommand{\url}[1]{#1}
\csname url@samestyle\endcsname
\providecommand{\newblock}{\relax}
\providecommand{\bibinfo}[2]{#2}
\providecommand{\BIBentrySTDinterwordspacing}{\spaceskip=0pt\relax}
\providecommand{\BIBentryALTinterwordstretchfactor}{4}
\providecommand{\BIBentryALTinterwordspacing}{\spaceskip=\fontdimen2\font plus
\BIBentryALTinterwordstretchfactor\fontdimen3\font minus
  \fontdimen4\font\relax}
\providecommand{\BIBforeignlanguage}[2]{{%
\expandafter\ifx\csname l@#1\endcsname\relax
\typeout{** WARNING: IEEEtran.bst: No hyphenation pattern has been}%
\typeout{** loaded for the language `#1'. Using the pattern for}%
\typeout{** the default language instead.}%
\else
\language=\csname l@#1\endcsname
\fi
#2}}
\providecommand{\BIBdecl}{\relax}
\BIBdecl

\bibitem{tekin2016adaptive}
C.~Tekin, J.~Yoon, and M.~van~der Schaar, ``Adaptive ensemble learning with
  confidence bounds for personalized diagnosis,'' in \emph{AAAI Workshop on
  Expanding the Boundaries of Health Informatics using AI (HIAI'16)}, 2016.

\bibitem{simons2008consumer}
D.~Simons, ``Consumer electronics opportunities in remote and home
  healthcare,'' \emph{Philips Research}, 2008.

\bibitem{cao2007intelligent}
Y.~Cao and Y.~Li, ``An intelligent fuzzy-based recommendation system for
  consumer electronic products,'' \emph{Expert Systems with Applications},
  vol.~33, no.~1, pp. 230--240, 2007.

\bibitem{beach2008whozthat}
A.~Beach, M.~Gartrell, S.~Akkala, J.~Elston, J.~Kelley, K.~Nishimoto, B.~Ray,
  S.~Razgulin, K.~Sundaresan, B.~Surendar \emph{et~al.}, ``{Whozthat? Evolving
  an ecosystem for context-aware mobile social networks},'' \emph{IEEE
  Network}, vol.~22, no.~4, pp. 50--55, 2008.

\bibitem{swix2004method}
S.~R. Swix, J.~R. Stefanik, and J.~C. Batten, ``Method and system for providing
  targeted advertisements,'' Apr.~6 2004, {US Patent 6,718,551}.

\bibitem{canzian2015timely}
L.~Canzian and M.~van~der Schaar, ``Timely event detection by networked
  learners,'' \emph{IEEE Trans. Signal Process.}, vol.~63, no.~5, pp.
  1282--1296, 2015.

\bibitem{eskicioglu2001overview}
A.~M. Eskicioglu and E.~J. Delp, ``An overview of multimedia content protection
  in consumer electronics devices,'' \emph{Signal Processing: Image
  Communication}, vol.~16, no.~7, pp. 681--699, 2001.

\bibitem{arsanjani2013improved}
R.~Arsanjani, Y.~Xu, D.~Dey, V.~Vahistha, A.~Shalev, R.~Nakanishi, S.~Hayes,
  M.~Fish, D.~Berman, G.~Germano \emph{et~al.}, ``Improved accuracy of
  myocardial perfusion spect for detection of coronary artery disease by
  machine learning in a large population,'' \emph{Journal of Nuclear
  Cardiology}, vol.~20, no.~4, pp. 553--562, 2013.

\bibitem{zheng2011attribute}
H.~Zheng, S.~R. Kulkarni, and H.~Poor, ``Attribute-distributed learning:
  models, limits, and algorithms,'' \emph{IEEE Trans. Signal Process.},
  vol.~59, no.~1, pp. 386--398, 2011.

\bibitem{littlestone1989weighted}
N.~Littlestone and M.~K. Warmuth, ``The weighted majority algorithm,'' in
  \emph{30th Annual Symposium on Foundations of Computer Science}, 1989, pp.
  256--261.

\bibitem{freund1995desicion}
Y.~Freund and R.~E. Schapire, ``A decision-theoretic generalization of on-line
  learning and an application to boosting,'' in \emph{European Conference on
  Computational Learning Theory}, 1995, pp. 23--37.

\bibitem{vapnik1992principles}
V.~Vapnik, ``Principles of risk minimization for learning theory,'' in
  \emph{Advances in Neural Information Processing Systems}, 1992, pp. 831--838.

\bibitem{kivinen2004online}
J.~Kivinen, A.~J. Smola, and R.~C. Williamson, ``Online learning with
  kernels,'' \emph{IEEE Trans. Signal Process.}, vol.~52, no.~8, pp.
  2165--2176, 2004.

\bibitem{auer2002adaptive}
P.~Auer, N.~Cesa-Bianchi, and C.~Gentile, ``Adaptive and self-confident on-line
  learning algorithms,'' \emph{Journal of Computer and System Sciences},
  vol.~64, no.~1, pp. 48--75, 2002.

\bibitem{chaudhuri2009parameter}
K.~Chaudhuri, Y.~Freund, and D.~J. Hsu, ``A parameter-free hedging algorithm,''
  in \emph{Advances in Neural Information Processing Systems}, 2009, pp.
  297--305.

\bibitem{de2014follow}
S.~De~Rooij, T.~Van~Erven, P.~D. Gr{\"u}nwald, and W.~M. Koolen, ``Follow the
  leader if you can, hedge if you must.'' \emph{Journal of Machine Learning
  Research}, vol.~15, no.~1, pp. 1281--1316, 2014.

\bibitem{cesa2006prediction}
N.~Cesa-Bianchi and G.~Lugosi, \emph{Prediction, learning, and games}.\hskip
  1em plus 0.5em minus 0.4em\relax Cambridge University Press, 2006.

\bibitem{auer2}
P.~Auer, N.~Cesa-Bianchi, Y.~Freund, and R.~Schapire, ``The nonstochastic
  multiarmed bandit problem,'' \emph{SIAM Journal on Computing}, vol.~32, pp.
  48--77, 2002.

\bibitem{cesa1997use}
N.~Cesa-Bianchi, Y.~Freund, D.~Haussler, D.~P. Helmbold, R.~E. Schapire, and
  M.~K. Warmuth, ``How to use expert advice,'' \emph{Journal of the ACM
  (JACM)}, vol.~44, no.~3, pp. 427--485, 1997.

\bibitem{auer}
P.~Auer, N.~Cesa-Bianchi, and P.~Fischer, ``Finite-time analysis of the
  multiarmed bandit problem,'' \emph{Machine Learning}, vol.~47, p. 235–256,
  2002.

\bibitem{lu2010contextual}
T.~Lu, D.~P{\'a}l, and M.~P{\'a}l, ``Contextual multi-armed bandits,'' in
  \emph{AISTATS}, 2010, pp. 485--492.

\bibitem{mangasarian1995breast}
O.~L. Mangasarian, W.~N. Street, and W.~H. Wolberg, ``Breast cancer diagnosis
  and prognosis via linear programming,'' \emph{Operations Research}, vol.~43,
  no.~4, pp. 570--577, 1995.

\bibitem{canzian2013ensemble}
L.~Canzian, Y.~Zhang, and M.~van~der Schaar, ``Ensemble of distributed learners
  for online classification of dynamic data streams,'' \emph{IEEE Trans. Signal
  and Information Processing over Networks}, vol.~1, no.~3, pp. 180--194, 2015.

\bibitem{blum1997empirical}
A.~Blum, ``Empirical support for winnow and weighted-majority algorithms:
  Results on a calendar scheduling domain,'' \emph{Machine Learning}, vol.~26,
  no.~1, pp. 5--23, 1997.

\bibitem{herbster1998tracking}
M.~Herbster and M.~K. Warmuth, ``Tracking the best expert,'' \emph{Machine
  Learning}, vol.~32, no.~2, pp. 151--178, 1998.

\bibitem{linqi2015health}
L.~Song, W.~Hsu, J.~Xu, and M.~van~der Schaar, ``Using contextual learning to
  improve diagnostic accuracy: Application in breast cancer screening,''
  \emph{IEEE J. Biomed. Health Inform.}, vol.~20, no.~3, pp. 902--914, 2016.

\bibitem{fan1999application}
W.~Fan, S.~J. Stolfo, and J.~Zhang, ``{The application of AdaBoost for
  distributed, scalable and on-line learning},'' in \emph{Proc. 5th ACM SIGKDD
  Int. Conf. Knowledge Discovery and Data Mining}, 1999, pp. 362--366.

\bibitem{masud2009integrating}
M.~M. Masud, J.~Gao, L.~Khan, J.~Han, and B.~Thuraisingham, ``Integrating novel
  class detection with classification for concept-drifting data streams,'' in
  \emph{Machine Learning and Knowledge Discovery in Databases}.\hskip 1em plus
  0.5em minus 0.4em\relax Springer, 2009, pp. 79--94.

\bibitem{street2001streaming}
W.~N. Street and Y.~Kim, ``{A streaming ensemble algorithm (SEA) for
  large-scale classification},'' in \emph{Proc. Seventh ACM SIGKDD Int. Conf.
  Knowledge Discovery and Data Mining}, 2001, pp. 377--382.

\bibitem{minku2012ddd}
L.~L. Minku and X.~Yao, ``{DDD: A new ensemble approach for dealing with
  concept drift},'' \emph{IEEE Trans. Knowl. Data Eng.}, vol.~24, no.~4, pp.
  619--633, 2012.

\bibitem{yubig2013}
Y.~Zhang, D.~Sow, D.~Turaga, and M.~van~der Schaar, ``A fast online learning
  algorithm for distributed mining of bigdata,'' \emph{ACM SIGMETRICS
  Performance Evaluation Review}, vol.~41, no.~4, pp. 90--93, 2014.

\bibitem{shutin2010space}
D.~Shutin, H.~Zheng, B.~H. Fleury, S.~R. Kulkarni, and H.~V. Poor,
  ``Space-alternating attribute-distributed sparse learning,'' in \emph{2nd
  International Workshop on Cognitive Information Processing (CIP)}, 2010, pp.
  209--214.

\bibitem{hershberger2001distributed}
D.~E. Hershberger and H.~Kargupta, ``Distributed multivariate regression using
  wavelet-based collective data mining,'' \emph{Journal of Parallel and
  Distributed Computing}, vol.~61, no.~3, pp. 372--400, 2001.

\bibitem{mateos2010distributed}
G.~Mateos, J.~A. Bazerque, and G.~B. Giannakis, ``Distributed sparse linear
  regression,'' \emph{IEEE Trans. Signal Process.}, vol.~58, no.~10, pp.
  5262--5276, 2010.

\bibitem{chen2012diffusion}
J.~Chen and A.~H. Sayed, ``Diffusion adaptation strategies for distributed
  optimization and learning over networks,'' \emph{IEEE Trans. Signal
  Process.}, vol.~60, no.~8, pp. 4289--4305, 2012.

\bibitem{nedic2009distributed}
A.~Nedic and A.~Ozdaglar, ``Distributed subgradient methods for multi-agent
  optimization,'' \emph{IEEE Trans. Autom. Control}, vol.~54, no.~1, pp.
  48--61, 2009.

\bibitem{nedic2010constrained}
A.~Nedic, A.~Ozdaglar, and P.~A. Parrilo, ``Constrained consensus and
  optimization in multi-agent networks,'' \emph{IEEE Trans. Autom. Control},
  vol.~55, no.~4, pp. 922--938, 2010.

\bibitem{yuan2011distributed}
D.~Yuan, S.~Xu, and H.~Zhao, ``Distributed primal--dual subgradient method for
  multiagent optimization via consensus algorithms,'' \emph{IEEE Transactions
  on Systems, Man, and Cybernetics, Part B (Cybernetics)}, vol.~41, no.~6, pp.
  1715--1724, 2011.

\bibitem{chang2014distributed}
T.-H. Chang, A.~Nedi{\'c}, and A.~Scaglione, ``Distributed constrained
  optimization by consensus-based primal-dual perturbation method,'' \emph{IEEE
  Trans. Autom. Control}, vol.~59, no.~6, pp. 1524--1538, 2014.

\bibitem{bubeck2010jeux}
S.~Bubeck, ``Jeux de bandits et fondations du clustering,'' Ph.D. dissertation,
  Universite Lille, 2010.

\bibitem{russo2014learning}
D.~Russo and B.~Van~Roy, ``Learning to optimize via posterior sampling,''
  \emph{Mathematics of Operations Research}, vol.~39, no.~4, pp. 1221--1243,
  2014.

\bibitem{abbasi2011improved}
Y.~Abbasi-Yadkori, D.~P{\'a}l, and C.~Szepesv{\'a}ri, ``Improved algorithms for
  linear stochastic bandits,'' in \emph{Advances in Neural Information
  Processing Systems}, 2011, pp. 2312--2320.

\bibitem{chlebus2009approximate}
E.~Chlebus, ``An approximate formula for a partial sum of the divergent
  p-series,'' \emph{Applied Mathematics Letters}, vol.~22, no.~5, pp. 732--737,
  2009.

\end{thebibliography}

\end{document}